\documentclass[letterpaper]{article}
\usepackage{uai2020}

\usepackage[margin=1in]{geometry}
\usepackage[round,comma]{natbib}

\usepackage{refcount}

\usepackage{hyperref}

\usepackage{color}
\usepackage{colortbl}
\usepackage[dvipsnames]{xcolor}

\usepackage{tikz}

\usepackage{times}

\usepackage{amsmath}

\DeclareMathOperator*{\argmax}{\arg\!\max}

\usepackage{amssymb}

\usepackage{amsthm}

\usepackage{todonotes}

\newtheorem{theorem}{Theorem}
\newtheorem{proposition}[theorem]{Proposition}
\newtheorem{definition}[theorem]{Definition}

\theoremstyle{remark}

\theoremstyle{definition}

\renewcommand{\thetheorem}{\arabic{theorem}}

\usepackage{mathtools}
\newcommand\numberthis{\addtocounter{equation}{1}\tag{\theequation}}

\ifx\ulesred\undefined
\usepackage{color}

\definecolor{ulesred}{rgb}{0.65,0.18,0.21} 

\newcommand{\ulesred}{\color{ulesred}}

\definecolor{darkgreen}{rgb}{0,0.5,0.2}

\definecolor{verylightgray}{gray}{0.85} 
\definecolor{mediumgray}{gray}{0.5} 
\definecolor{darkgray}{gray}{0.4} 

\definecolor{lavender}{cmyk}{0,0.5,0,0}

\definecolor{darkblue}{rgb}{0,0,0.5}

\definecolor{orange}{rgb}{1,0.5,0}

\else

\fi

\ifx\black\undefined
\newcommand{\black}{\color{black}}
\fi

\ifx\blue\undefined
\newcommand{\blue}{\color{blue}}
\fi

\ifx\red\undefined
\newcommand{\red}{\color{red}}
\fi

\ifx\green\undefined
\newcommand{\green}{\color{green}}
\fi

\ifx\white\undefined
\newcommand{\white}{\color{white}}
\fi

\ifx\magenta\undefined
\newcommand{\magenta}{\color{magenta}}
\fi 

\newcommand{\Xcal}{\mathcal{X}}
\newcommand{\Ycal}{\mathcal{Y}}

\newcommand{\accuracy}[1]{\text{acc}_{#1}}
\newcommand{\information}[1]{\text{I}_{#1}}

\title{Specialists Outperform Generalists in Ensemble Classification}

\author{ {\bf Sascha Meyen}$^1$ \\
\And
{\bf Frieder Göppert}$^1$  \\
\And
{\bf Helen Alber}$^1$   \\ \\
$^1$ Department of Computer Science, University of Tübingen, Tübingen, Germany \\
$^2$ Max Planck Institute for Intelligent Systems, Tübingen, Germany \\
\And
{\bf Ulrike von Luxburg}$^{1, 2}$   \\
\And
{\bf Volker H. Franz}$^1$   \\ \\
}

\begin{document}

\maketitle

\begin{abstract}

Consider an ensemble of $k$ individual classifiers whose accuracies are known. Upon receiving a test point, each of the classifiers outputs a predicted label and a confidence in its prediction for this particular test point. 
In this paper, we address the question of whether we can determine the accuracy of the ensemble. 
Surprisingly, even when classifiers are combined in the statistically optimal way in this setting, the accuracy of the resulting ensemble classifier cannot be computed from the accuracies of the individual classifiers---as would be the case in the standard setting of confidence weighted majority voting. 
We prove tight upper and lower bounds on the ensemble accuracy. We explicitly construct the individual classifiers that attain the upper and lower bounds: specialists and generalists. Our theoretical results have very practical consequences: 
(1) If we use ensemble methods and have the choice to construct our individual (independent) classifiers from scratch, then we should aim for specialist classifiers rather than generalists. (2) Our bounds can be used to determine how many classifiers are at least required to achieve a desired ensemble accuracy. Finally, we improve our bounds by considering the mutual information between the true label and the individual classifier's output.





\end{abstract}


\section{Introduction}
Suppose a black-box classifier returns a prediction along with a confidence value indicating the probability that this prediction is correct. For example, a deep neural network may take an image of a patient's retina and predict whether the patient suffers from a retinal disease (example taken from \citealp{ayhan2020expert,leibig2017leveraging,ayhan2018test}). It also outputs a confidence in this prediction based on the particular retina image. Suppose we apply $k$ black-box classifiers each receiving its own retina image as input, observe their individual prediction-confidence output pairs and combine them into an ensemble classifier. In this paper, we investigate the question: {\bf Given that the individual black-box classifiers output predicted labels together with confidences, what can we say about the accuracy achieved by the ensemble classifier? And can we characterize which type of individual classifier leads to a better vs. worse ensemble performance?}

At first glance, this problem seems trivial. If we know that all individual classifiers have the same accuracy, then the best we can do is a \emph{majority vote} (MV; \citealp{grofman1983thirteen,de2014essai}). This is common practice in many ensemble approaches in machine learning, for example in random forests \citep{breiman2001random}. If some of the individual classifiers are known to have a higher accuracy than others, they should receive a higher weight. Based on this knowledge, the best we can do is \emph{confidence weighted majority voting} (CWMV; see \citealp{nitzan1982optimal,einhorn1977quality}), where the confidence in a classifier is derived from its overall accuracy. In both cases, the accuracy of the ensemble classifier is well known and can be computed from the accuracies of the individual classifiers (under mild assumptions such as conditional independence).

However, these approaches do not fully capture the retina example from the beginning because they do not consider the classifier's ``local confidences'': For each image, the classifier produces a confidence in its prediction for this particular image.
And this is where it gets interesting: Instead of using CWMV, where the confidence is based on the overall accuracy of the classifier, we get better classification results by using the local confidences for each prediction. Somewhat surprisingly, in this setting we can no longer compute exactly what the resulting ensemble accuracy is going to be. On the contrary. We prove in Section~\ref{sec:IndAccGroupAcc} that, depending on the distribution of confidence values, there is a whole range of ensemble accuracies that can occur. 
\textbf{Our contribution} is to derive lower and upper bounds on the ensemble accuracy in this setting. This is interesting if we want to determine how many classifiers (each requiring an independently drawn, potentially costly retina image) are needed to guarantee a certain ensemble accuracy. From our proofs, we derive guiding principles on how to construct ideal classifiers for an ensemble: We will show that it is better to include ``specialist'' classifiers that are particularly good on some instances and close to random guessing on others rather than to include ``generalist'' classifiers that are moderately good on all instances. This is true for independent specialists that did not coordinate to specialize on distinct subsets of the input space. 

In the second part (Section~\ref{sec:Information}), we additionally look into the mutual information as an indicator for the effectiveness of a classifier in ensembles. We provide better bounds on the possible ensemble accuracies. Even when classifiers have the same accuracy, they can differ in how much information they provide about the true label and therefore differ in their contribution to the ensemble.



Of course, ensemble methods are abundant in machine learning and statistics, just consider random forests, bagging and boosting as examples. Compared to these lines of work, our approach starts from the other end. Rather than explicitly training certain ensembles, we are looking for generic building principles for ensembles. We build on the setting of probability elicitation \citep{degroot1983comparison,masnadi2013refinement} and ask the question: Which possible ensemble accuracies can be achieved by a set of individual classifiers with known accuracies, and which kind of individual classifiers produce the best- and worst-case ensembles?


\section{Setup, notation, and background}
\label{sec:notation}

\subsection{Individual classifiers and confidences}

\begin{figure}[!ht]
\includegraphics[width=\linewidth]{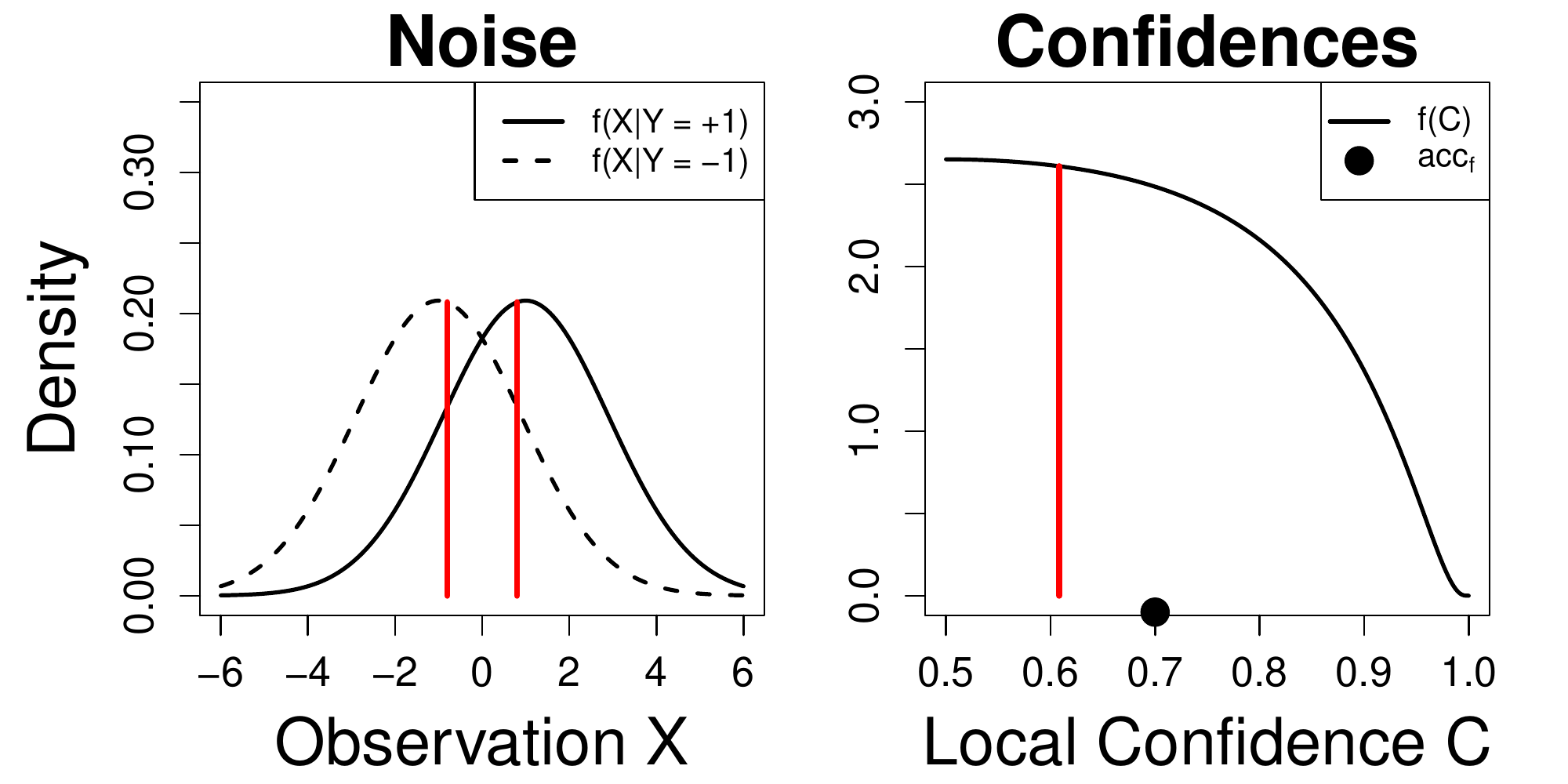}
\caption{
\textbf{Example of a confidence distribution.} 
A classification setting with normal noise distribution (left) is mapped to a confidence distribution (right). For normally distributed noise with $\sigma = 2.1$, an observation of $X = 0.8$ (indicated by the red bar) corresponds to a local confidence of $C(X) = 61\%$ such that an individual classifier outputs $(\hat{Y} = +1, C = 61\%)$. Therefore, the density at $X=0.8$ (plus that at $X=-0.8$, because it produces the same confidence albeit predicting $\hat{Y}=-1$) corresponds to the density at confidence $C=61\%$ on the right. 
The overall accuracy of the black-box classifier in this setting is $\accuracy{f} = 70\%$, indicated by the black dot. 
}
\label{fig:normalExample}
\end{figure}

We work in a standard classification setting with input points $X$ in some abstract input space $\Xcal$, binary labels $Y \in \{-1,1\}$, and a joint probability distribution $P$ on $\Xcal \times \Ycal$. We assume that both labels have the same probability, $P(Y=+1) = P(Y=-1) = 0.5$, meaning that, in our retina example, patients equally often have the disease as they do not have the disease (we make this assumption to keep the notation simple but our results may be generalized to a setting with unequal probabilities). A black-box classifier, upon observing a test point $X \in \Xcal$ (a retina image), outputs two quantities: the predicted label $\hat{Y} \in \{-1,1\}$ and the pointwise, or local, confidence $C \in [0.5,1]$. We assume Bayes classifiers so that the predictions are optimal, $\hat{Y} = \argmax_{y\in\{-1,+1\}}P(Y|X)$. Furthermore, we assume the classifiers to be perfectly calibrated. That is, the local confidence is exactly the probability of that particular prediction to be correct, $C(X) = P(Y = \hat{Y}|X)$, or short, $C$. (The range of possible ensemble accuracies would be even larger if we dropped this calibration assumption.) We call $C$ a \emph{local} confidence to stress that it is different for each input point $X$ whereas we use the term \emph{accuracy} to refer to the overall probability of a classifier making a correct prediction across the input space.

In the following, we will describe a black-box classifier by its \textbf{confidence distribution} $f(C)$. A confidence distribution $f$ is a probability distribution on $[0.5,1]$ that describes how often each local confidence $C$ is sampled. In Figure~\ref{fig:normalExample}, we show how a confidence distribution is related to a classification setting with normal noise. In our retina example, a classifier's confidence distribution describes how often we get retina images of a certain quality such that a classifier can make a prediction with confidence $C$. We assume the confidences to be independent of the true label, $f(C|Y) = f(C)$ meaning that the quality of a retina image is independent of whether the patient has the disease or not. 

If the confidence distribution $f$ of a classifier is known, its classification accuracy can be computed from $f$:
\begin{align*}
    \accuracy{f} := P(Y = \hat{Y}) = \int_{0.5}^1 f(c) \cdot c~dc \text{~.} \label{eqn:PiFromC} \numberthis
\end{align*}
However, in this paper, we deal with the more realistic scenario where the underlying local confidence distribution~$f$ is unknown and we only know the accuracy $\accuracy{f}$ of an individual classifier.

\subsection{Ensemble classifiers}
\label{sec:notationEnsemble}

We obtain an ensemble prediction by optimally combining the outputs of $k$ individual classifiers. The individual classifiers have unknown confidence distributions $f_1$, $f_2$, ..., $f_k$. We will make the important assumption that these confidence distributions are pairwise independent, $\forall i\neq j: f_i \bot f_j$. In our example, this means that the quality of retina images is independently drawn for each classifier (from its unknown confidence distributions). Under this assumption, the individual confidence distributions combine into the ensemble confidence distribution, denoted by $f_e$ (see Section~\ref{subsec:lCWMV}). Since we do not know the individual confidence distributions, $f_1$, $f_2$, ... $f_k$, we also do not know the exact ensemble confidence distribution, $f_e$.

The \textbf{goal of this paper} is to determine the accuracy of that ensemble classifier, $\accuracy{f_e}$, given that we only know the accuracies of the individual classifiers $\accuracy{f_1}$, $\accuracy{f_2}$, ... $\accuracy{f_k}$ but not their exact confidence distributions, and to characterize which type of individual classifier leads to a better / worse ensemble performance.

\section{Confidence Weighted Majority Voting}
\label{subsec:cwmv}

In this section, we first recap the traditional approach of CWMV and then introduce our modification based on local confidences, which we call $l$CWMV. 

\subsection{Traditional approach: CWMV}
\label{subsec:gCWMV}

In the traditional setting of CWMV \citep{grofman1983thirteen,nitzan1982optimal}, upon receiving input, a classifier outputs a prediction $\hat{Y}$, but not the local confidence for the particular test point. All we know is the (global) accuracy of the black-box classifier. In an ensemble, we observe a set of $k$ predictions $\hat{Y}_1$, $\hat{Y}_2$, ..., $\hat{Y}_k$ from classifiers with accuracies $\accuracy{1}$, $\accuracy{2}$, ..., $\accuracy{k}$. It has been proven \citep{grofman1983thirteen} that the optimal way to form an ensemble prediction in this scenario is to weight the individual classifiers' votes based on their accuracies, $W_i = \log (\accuracy{i}/(1-\accuracy{i}))$. These weights are therefore based on the overall accuracies of the individual classifiers. Traditional CWMV then produces the optimal ensemble prediction $\hat{Y}_e$ and the ensemble confidence in that prediction $C_e$ as 
\vspace{-0.2cm}
\begin{align*}
    &\hat{Y}_e = \text{sign} \left( \sum_{i=1}^k W_i \hat{Y}_i \right)\text{, and} \numberthis \label{eqn:CWMV_Y} \\
    &C_e = \left( 1 + \exp \left(- \Biggl| \sum_{i=1}^k W_i \hat{Y}_i \Biggl| \right) \right)^{-1}\text{.} \numberthis \label{eqn:CWMV_C}
\end{align*}
Note that when $\accuracy{i} = 1$ for any $i$, the weight $W_i$ is undefined and therefore $\hat{Y}_e$ and $C_e$ are set to $\hat{Y}_e = \hat{Y}_i$ and $C_e = 1$ by convention because classifier $i$ is always correct in its prediction. 

\subsection{Modification with local confidences: $l$CWMV} 
\label{subsec:lCWMV}

We modify the traditional setting such that, upon receiving input point $X$, a classifier outputs its prediction $\hat{Y}$ together with a local confidence $C(X)$. 
It is straightforward to see that the optimal combination of the outputs of $k$ classifiers, $i \in \{1..k\}: \left(\hat{Y}_i, C(X_i)\right)$, will base the weights not on the accuracies of the individual classifiers but on their local confidences for their individual input points: $W(X_i) = \log (C(X_i)/(1-C(X_i)))$. The ensemble prediction $\hat{Y}_e$ and confidence $C_e$ are then computed analogously to Equations~\eqref{eqn:CWMV_Y} and \eqref{eqn:CWMV_C}, using the local weights $W_i = W(X_i)$. While the weights were constant in traditional CWMV, they can differ from prediction to prediction in $l$CWMV.
 
In contrast to the traditional approach, we can no longer compute the ensemble accuracy, $\accuracy{f_e}$, based on the accuracies, $\accuracy{1}$, $\accuracy{2}$, ..., $\accuracy{k}$. Only when the exact distributions over the local confidences, $f_1$, $f_2$, ..., $f_k$, are known, we can derive the confidence distribution of the ensemble, $f_e$, and thereupon the ensemble accuracy $\accuracy{f_e}$. In the following, we denote the operation of combining individual confidence distributions in the $l$CWMV setting by $\otimes$ (formally defined in the Supplementary Material Section \ref{sup:convolution}) so that the ensemble confidence distribution is denoted by $f_e := \otimes_{i=1}^k f_i$ . 

In many practical examples, the confidence distribution of individual classifiers will not be known. Especially in cases where there is a high cost for obtaining predictions, as in our retina example, estimating confidence distributions is expensive. Not knowing the individual classifier's confidence distribution but only their overall accuracies entails some uncertainty about the ensemble confidence distribution $f_e$. Consequently, there are different possible values for the ensemble accuracy, $\accuracy{f_e}$. The question we now answer is: What are the best and worst ensemble accuracies that can be achieved? And which individual confidence distributions contribute more to the ensemble accuracy than others?


\section{Individual accuracies do not uniquely determine ensemble accuracy}
\label{sec:IndAccGroupAcc}

In this section, we provide bounds on the ensemble accuracy when only the accuracies of the individual classifiers are known. Numerical examples can be found in \url{https://osf.io/mvsgh/}. The relevant aspect of a classifier will be its confidence distribution, $f(C)$, which is only constraint by the given individual accuracy. The classifiers that produce the best- and worst-case ensemble accuracies will be called specialists and generalists. A specialist and a generalist, even when they have the same accuracy, behave very differently in ensembles due to their different confidence distributions, see Figure~\ref{fig:classifierDistributions} (top). Because these two extreme classifiers produce only a discrete amount of confidence levels, we will denote their probability distributions as weighted sums of Dirac probability measures $\delta_c$ that have point mass 1 at point $c$. 

Intuitively, a classifier is a specialist if it achieves high confidence on some parts of the input space while it is close to random guessing on the rest. Formally, it outputs predictions with confidence either $C=50\%$ (random guessing) or $C=100\%$ (absolute certainty), see Figure~\ref{fig:classifierDistributions} (top left). The proportion of these two cases determines the overall accuracy of the specialist. 
\begin{definition}
    \label{def:specialist}
    \textbf{(Specialist)}
    A binary black-box classifier with accuracy $\accuracy{}$ is called specialist if its confidence distribution is given by 
    \begin{align*}
        f^\text{specialist}_{\accuracy{}} = w_{0.5}\delta_{0.5} + w_{1}\delta_{1} \text{,}
    \end{align*}
    with constants $w_{0.5} = 2(1-\accuracy{})$ and $w_{1} = 2(\accuracy{}-0.5)$.
\end{definition}

\begin{figure}[!t]
\includegraphics[width=\linewidth]{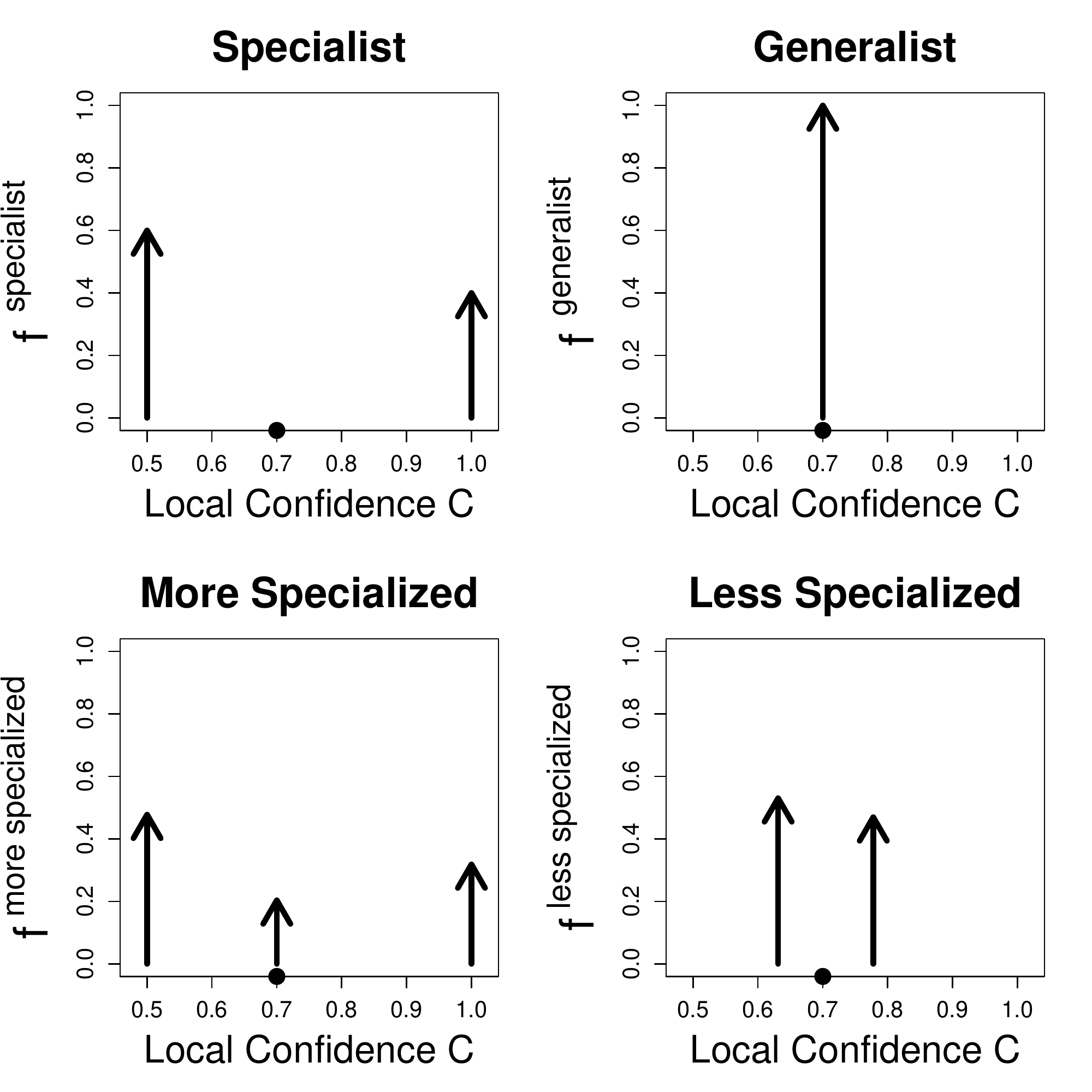}
\caption{{\bf Classifiers' confidence distributions that provide best- and worst-case ensemble accuracies.} Top row: Confidence distribution of a corresponding specialist and generalist classifier with $\accuracy{f} = 70\%$ (black dot). The confidence distributions consist of point masses as indicated by the arrows . Bottom row: Confidence distributions of more and less specialized classifiers with $\accuracy{f} = 70\%$ and $\information{f} = 0.25$~bit. }
\label{fig:classifierDistributions}
\end{figure}

\begin{figure*}[!ht] 
\includegraphics[width=\linewidth]{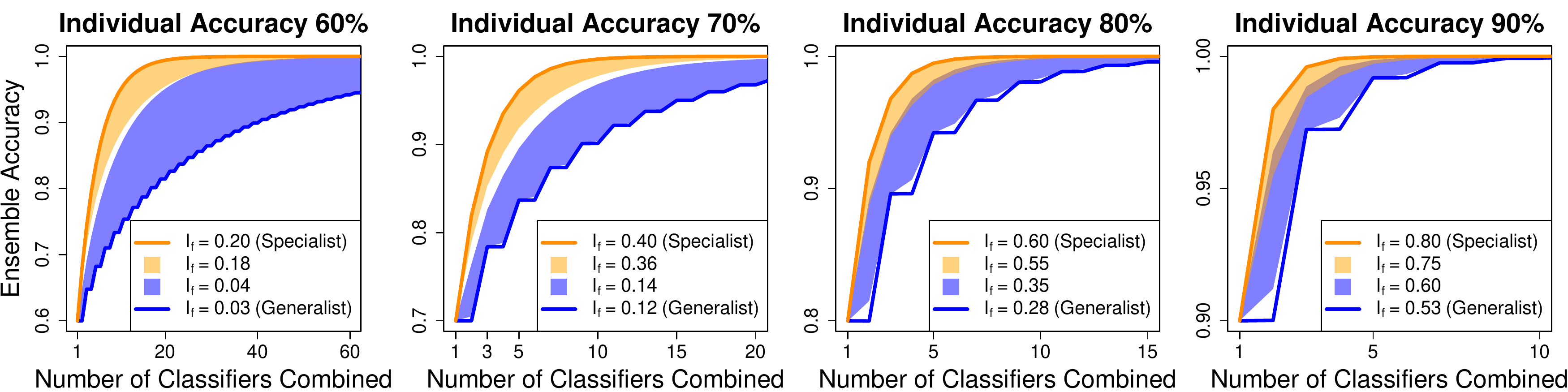}
\caption{ {\bf Illustrations of the bounds in Theorems \ref{theorem:IndAccGroupAcc} and \ref{theorem:IndAccIndInfoGroupAcc}}. 
Each plot shows the ensemble accuracy $\accuracy{f_e}$ as a function of the number $k$ of individual classifiers. Within each subplot, all individual classifiers have the same individual accuracy $\accuracy{i}$ as indicated in the title of the subplot. Best- and worst-case ensemble accuracies according to Theorem~\ref{theorem:IndAccGroupAcc} are shown as solid orange and blue lines, achieved by the two extreme cases, specialists and generalists, which also have the highest resp. lowest information (Theorem~\ref{prop:IndAccIndInfo}). 
When in addition to the accuracy the information $I_f$ of the individual classifiers is known, the range of possible ensemble accuracies gets smaller (Theorem~\ref{theorem:IndAccIndInfoGroupAcc}): For example, individual classifiers with a relatively high information produce a range of possible ensemble accuracies (light orange shaded areas) exceeding that of classifiers with low information (light blue shaded areas). 
The steps in the lower bound result from the problem of conducting majority votes in ensembles of even size (no tiebreakers).}
\label{fig:panel1}
\end{figure*}

Generalists, on the other hand, work equally well on all of the input space. Their confidence is constant with $C=\accuracy{}$, see Figure~\ref{fig:classifierDistributions} (top right).
\begin{definition}
    \label{def:generalist}
    \textbf{(Generalist)}
    A binary black-box classifier with accuracy $\accuracy{}$ is called generalist if its confidence distribution is given by
    \begin{align*}
        f^\text{generalist}_{\accuracy{}} = \delta_{\accuracy{}} \text{.}
    \end{align*}
\end{definition}

The following theorem states that generalists and specialists are the worst and best case classifiers when used in an ensemble. 
\begin{theorem}
    \label{theorem:IndAccGroupAcc}
    \textbf{(Specialists and generalists bound the ensemble accuracy)}
    Consider $k$ classifiers with individual accuracies $\accuracy{i}$ and confidence distributions $f_i$ ($i \in \{1..k\}$). For each $i$, let 
    $f^\text{generalist}_{\accuracy{i}}$ and $f^\text{specialist}_{\accuracy{i}}$ be a generalist resp. specialist classifier that has the same accuracy as classifier $i$.
    Now consider the ensemble classifier based on the original classifiers with ensemble confidence distribution $f_e~=~\bigotimes_{i = 1}^k\,f_i$ according to $l$CWMV as well as the ensemble of generalists and ensemble of specialists with ensemble confidence distributions $f^\text{generalist}_e~=~\bigotimes_{i = 1}^k\,f^\text{generalist}_{\accuracy{i}}$ and $f^\text{specialist}_e~=~\bigotimes_{i = 1}^k\,f^\text{specialist}_{\accuracy{i}}$. Then the accuracy of the original ensemble is lower and upper bounded by the accuracies of the generalist and specialist ensembles: 
    $$\accuracy{f^\text{generalist}_e} \leq \accuracy{f_e} \leq \accuracy{f^\text{specialist}_e}~.$$
\end{theorem}
The formal proof of this theorem is in the Supplementary Materials, Section~\ref{sup:IndAccEnsembleAcc}. The proof idea is that merging confidence distributions makes the ensemble accuracy worse: When a classifier does not distinguish between high vs. low confidence cases (Figure~\ref{fig:classifierDistributions} top left) and instead always outputs an average confidence (Figure~\ref{fig:classifierDistributions} top right), the ensemble is less effective in weighing that classifier's predictions. It helps to know which predictions should be taken into account (high confidence cases) and which should be disregarded (low confidence cases). Distinguishing between high and low confidence cases is related to the concept of refinement \citep{degroot1983comparison,masnadi2013refinement}, see Supplementary Materials, Section~\ref{sup:refinement}. In consequence, the best ensemble accuracy comes from the most refined confidence distributions (specialists); and the worst ensemble accuracy comes from the least refined confidence distributions (generalists). Even though confidences can vary from prediction to prediction in our $l$CWMV setting, generalists do not make use of this possibility and always output the same confidence. They receive a constant weight as in the traditional CWMV setting. Therefore, our lower bound for the ensemble accuracy corresponds to the behavior of traditional CWMV. 

To get an intuition for the meaning of the theorem, consider again the retina example. Assume we have $k=3$ classifiers that take independently drawn retina images from a patient and return predictions as well as confidences. Let their predictions be correct with accuracies $\accuracy{1} = \accuracy{2} = \accuracy{3} = 70\%$ (in Figure~\ref{fig:panel1}, second plot). Then, if these classifiers are generalists, their ensemble accuracy will be 78\% (blue lower bound). But if they are specialists, their ensemble accuracy will be 89\% (orange upper bound)---a large range that makes a crucial difference in practice.
If the three classifiers' confidence distributions are not that of specialists or generalists (as in Figure~\ref{fig:classifierDistributions}) but an intermediate case as in our normal noise example (Figure~\ref{fig:normalExample}, right) the ensemble accuracy is in between the bounds, here, at 82\%.
See Figure~\ref{fig:panel1} for more numerical examples. We only show cases in which the individual accuracies are equal but our theorems can be applied to classifiers with different individual accuracies.

Theorem \ref{theorem:IndAccGroupAcc} carries two important messages: (1) Even when we know the accuracies of the individual classifiers and we combine their output in the statistically optimal way (with $l$CWMV), we are far from being able to predict the ensemble accuracy (unless we know the confidence distributions). (2) When we use ensemble methods and have the choice to construct our individual classifiers from scratch, then we should aim for specialist classifiers rather than generalists.

Crucially, \textbf{even without coordination between the classifiers, specialization is advantageous.} Specialists' confidence distributions are, by assumption, independent. Specialists do not divide the input space by specializing on separate regions. In our retina example, it is \textit{not} the case that one specialist is trained on one subtype of retinal disease while a different specialist is trained on another subtype. This would contradict our assumption that individual confidences are independently drawn ($\forall i,j\in\{1..k\}: f_i \bot f_j$, introduced in Section~\ref{sec:notationEnsemble}). When one specialist classifier achieves a high confidence it is \emph{not} more likely that the other specialists produce a low confidence as it would be the case when they had separate specializations. This highlights the effectiveness of specialists even in independent ensembles.


\section{Better bounds for ensemble accuracy with mutual information}
\label{sec:Information}

As shown, the range of possible accuracies of ensemble classifiers outlined in  Theorem~\ref{theorem:IndAccGroupAcc} can be large. In this section, we improve the bounds to better predict what the ensemble accuracy will be. We will assume that another performance measure next to the individual classifier's accuracy is known: the mutual information between the true label and the individual classifier's output \citep{shannon1948mathematical,cover2006elements,mackay2003information}. This is just one alternative quality measure of the classifier, and many more such scoring functions exist (see \citealp{masnadi2013refinement,masnadi2017combining}). We choose the mutual information for its natural properties but our results can be transferred to other convex scoring functions.

\subsection{Mutual information measures effectiveness in ensembles}
\label{sec:IndAccIndInf}

In addition to the accuracy of a classifier, we consider the mutual information $I$ between the true label $Y$ and the classifier's output $O = (\hat{Y}, C)$, which is $I(Y;O) = H(Y) - H(Y|O)$, where $H$ denotes the (conditional) entropy of a random variable. With some simple rearrangement (see Supplementary Material, Section \ref{sup:rearrangingMutualInformation}), the mutual information can be shown to only depend on the classifiers' confidence distribution $f$:
\begin{align*}
    \information{f} := I\left(Y; O \right) = \int_{0.5}^1 f(c) \cdot \left(1 - H_2(c)\right)~dc \text{,}  \label{eqn:IotaFromC} \numberthis
\end{align*}
where $H_2$ is the binary entropy, $H_2(c) = c\log_2\left(\frac{1}{c}\right) + (1-c)\log_2\left(\frac{1}{1-c}\right)$ for $c\in [0.5,1]$. In the following, we will denote a classifier's \emph{information} by $\information{f}$, analogously to its accuracy $\accuracy{f}$, as a performance measure based on a classifier's confidence distribution $f$.
For classifiers with fixed accuracy $\accuracy{f}$, specialists have the highest possible information and generalists have the lowest possible information.
\begin{proposition}
    \label{prop:IndAccIndInfo}
    \textbf{(Specialists and generalists bounds the individual information)}
    A classifier with confidence distribution $f$ and accuracy $\accuracy{}$ has an information between
    \begin{equation*}
        \information{f^\text{generalist}_{\accuracy{}}} \leq \information{f} \leq \information{f^\text{specialist}_{\accuracy{}}}.
    \end{equation*}
\end{proposition}
The proof is in the Supplementary Materials, Section~\ref{sup:IndAccBoundsIndInfo}. In our example, when an individual classifier has an accuracy of $\accuracy{} = 70\%$, its transmitted information lies between $0.12$--$0.4$~bit, depending on its confidence distribution. With this, all classifiers can be described by two values, their accuracy $\accuracy{}$ and information $\information{}$, and these values lie in the shaded area in Figure~\ref{fig:proofIdea} (middle): Higher accuracy (along the x-axis) loosely coincides with higher information (y-axis) but this is no one-to-one relation.

\begin{figure*}[t]
\includegraphics[width=\linewidth]{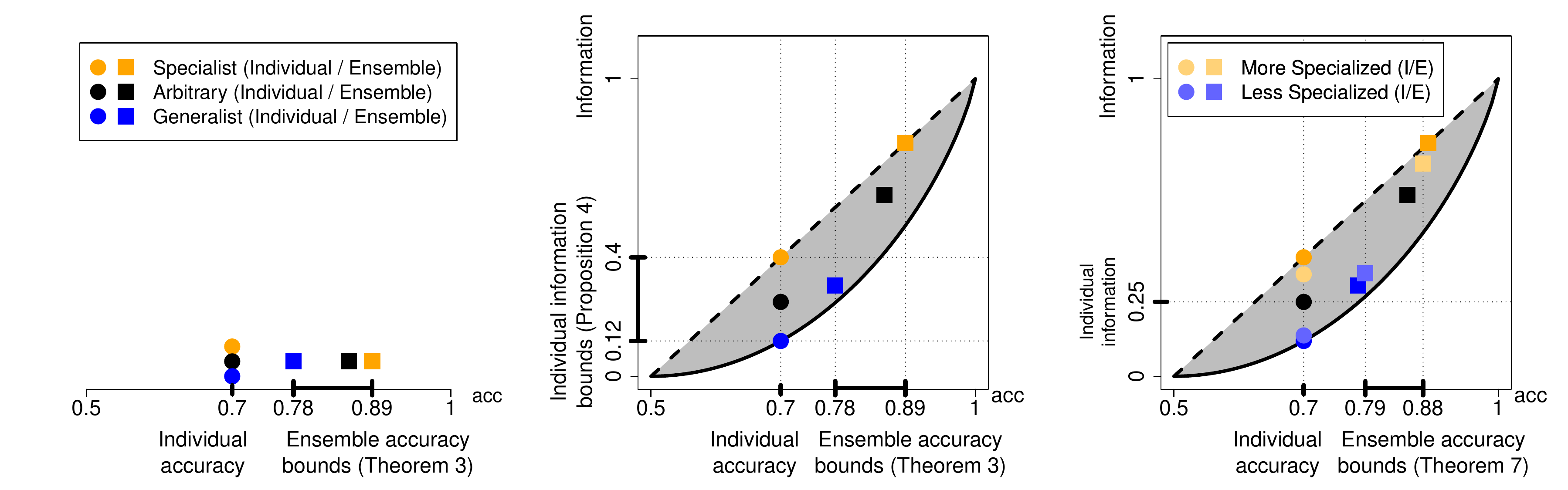}
\caption{\textbf{Bounding ensemble accuracy with Theorem~\ref{theorem:IndAccGroupAcc} and \ref{theorem:IndAccIndInfoGroupAcc}.} 
Left: Consider an individual classifier with an arbitrary and unknown confidence distribution $f$. We depict its known individual accuracy by a black dot (in this example, $\accuracy{f} =70\%$). We construct corresponding specialist and generalist with accuracies marked by blue and orange dots. Together, $k=3$ arbitrary classifier with $\accuracy{f} =70\%$ form an ensemble with ensemble accuracy (depicted by the black square) bounded by that of a generalist ensemble (blue square) and specialist ensemble (orange square), see Theorem~\ref{theorem:IndAccGroupAcc}. 
Middle: In addition to the accuracy of classifiers (x-axis) we consider their information (y-axis). The information losely depends on the accuracy $\accuracy{f}$ as marked by the grey area, see Proposition~\ref{prop:IndAccIndInfo}. Generalists lie on the lower solid line and specialists lie on the upper dashed line. Thus, any arbitrary classifier's accuracy-information pair, ($\accuracy{f}$, $\information{f}$), lies in the grey crescent shape. Again we depict the three individual classifiers of the left figure by dots and the corresponding ensemble classifiers by squares. 
Right: When the individual classifiers' information is known (here, $0.25$~bit), Theorem~\ref{theorem:IndAccIndInfoGroupAcc} provides better bounds. Less specialized (light blue dot) and more specialized classifiers (light orange dot) form ensembles (same colored squares) whose accuracy bounds the ensemble accuracy of arbitrary classifiers. In this example, bounds from Theorem~\ref{theorem:IndAccGroupAcc} improve only slightly but see Figure~\ref{fig:panel1}.
}
\label{fig:proofIdea}
\end{figure*}

\subsection{Improved ensemble accuracy bounds}
\label{sec:IndAccInfGrouPAcc}

We will now assume that we know both, the accuracy and the information of the individual classifiers. Given these two measures, we can provide better bounds on the ensemble accuracy. These two measures still do not uniquely determine the confidence distribution of a classifier so that different ensemble accuracies are possible. As before (with specialists and generalists), we construct two confidence distributions: the more specialized classifier and the less specialized classifier. They will provide the new bounds.

The more specialized classifier (to a given accuracy $\accuracy{}$ and information $\information{}$) is a mixture of specialist and generalist producing confidences at $C=0.5$, $C=\accuracy{}$ and $C=1$, see Figure~\ref{fig:classifierDistributions} (bottom left). The weights are such that the more specialized classifier can be shown to improve the ensemble accuracy.
\begin{definition}
    \label{def:moreSpecialized}
    \textbf{(More specialized classifier)}
    A binary black-box classifier to the accuracy $\accuracy{}$ and information $\information{}$ is called more specialized if its confidence distribution is given by
    \begin{align*}
        f^\uparrow_{\accuracy{},\information{}} = w_{0.5}\delta_{0.5} + w_{\accuracy{}}\delta_{\accuracy{}} + w_{1}\delta_{1} \text{,}
    \end{align*}
    with constants 
    $w_{0.5}~=~ \frac{2(1-\accuracy{})(\information{}+g - 1 + H_2(\accuracy{}))}{2\accuracy{} - 2 + H_2(\accuracy{})}$, 
    $w_{\accuracy{}}~=~\frac{2\accuracy{} - 1 - (\information{}+g)}{2\accuracy{} - 2 + H_2(\accuracy{})}$ and 
    $w_{1}~=~\frac{2(\accuracy{}-0.5)(\information{}+g - 1 + H_2(\accuracy{}))}{2\accuracy{} - 2 + H_2(\accuracy{})}$. Constant $g$ is defined in the Supplementary Material, Section~\ref{sup:IndAccIndInfoEnsembleAcc}.\\
\end{definition}

The less specialized classifier is similar to a generalist but it can distinguish between slightly below average ($C=c^\text{l}$) and slightly above average confidences ($C=c^\text{r}$), see Figure~\ref{fig:classifierDistributions} (bottom right).
\begin{definition}
    \label{def:lessSpecialized}
    \textbf{(Less specialized classifier)}
    A binary black-box classifier to the accuracy $\accuracy{}$ and information $\information{}$ is called more specialized if its confidence distribution is given by
    \begin{align*}
        f^\downarrow_{\accuracy{},\information{}} = w_{c^\text{l}}\delta_{c^\text{l}} + w_{c^\text{r}}\delta_{c^\text{r}} \text{,}
    \end{align*}
    with constants 
    $c^\text{l}~=~\frac{2(\accuracy{}-0.5)(\accuracy{}-\information{}) - (1-\accuracy{})(1-H_2(\accuracy{}))}{2(\accuracy{}-0.5)(1-\information{}) - 2(1-\accuracy{})(1-H_2(\accuracy{}))}$,
    $c^\text{r}~=~\frac{2(\accuracy{}-0.5)(\accuracy{}-1+H_2(\accuracy{}))-(1-\accuracy{})\information{} }{2(\accuracy{}-0.5)H_2(\accuracy{}) - 2(1-\accuracy{})\information{}}$, 
    as well as $w_{c^\text{l}} = \frac{c^\text{r} - \accuracy{}}{c^\text{r} - c^\text{l}}$ and $w_{c^\text{r}} = \frac{\accuracy{} - c^\text{l}}{c^\text{r} - c^\text{l}}$.
\end{definition}

We can bound the ensemble accuracy of classifiers with known accuracies and information by the ensemble accuracies of more resp. less specialized classifiers.
\begin{theorem}
    \label{theorem:IndAccIndInfoGroupAcc}
    \textbf{(More and less specialized classifiers bound the ensemble accuracy)}
    Consider $k$ classifiers with individual accuracies $\accuracy{i}$, individual information $\information{i}$ and confidence distributions $f_i$ ($i \in \{1..k\}$). For each $i$, let 
    $f^\downarrow_{\accuracy{i},\information{i}}$ and $f^\uparrow_{\accuracy{i},\information{i}}$ be the less resp. more specialized classifier constructed to the accuracy and information of classifier $i$. 
    Now consider the ensemble classifier based on the original classifiers with ensemble confidence distribution $f_e~=~\bigotimes_{i = 1}^k\,f_i$ according to $l$CWMV as well as the ensemble of less and more specialized classifiers with ensemble confidence distributions $f^\downarrow_e~=~\bigotimes_{i = 1}^k\,f^\downarrow_{\accuracy{i},\information{i}}$ and $f^\uparrow_e~=~\bigotimes_{i = 1}^k\,f^\uparrow_{\accuracy{i},\information{i}}$. Then the accuracy of the original ensemble is lower and upper bounded by the accuracies of the less and more specialized ensembles: 
    $$\accuracy{f^\text{generalist}_e} \leq \accuracy{f^\downarrow_e} \leq \accuracy{f_e} \leq \accuracy{f^\uparrow_e} \leq \accuracy{f^\text{specialist}_e}.$$    
\end{theorem}

The proof is in the Supplementary Material, Section~\ref{sup:IndAccIndInfoEnsembleAcc}. The proof idea is visualized in Figure~\ref{fig:proofIdea}. Theorem~\ref{theorem:IndAccIndInfoGroupAcc} shows that additionally knowing the information of the individual classifiers allows to predict the ensemble performance better than when only their accuracy is known, see Figure~\ref{fig:panel1}. In the retina example, if we know that the $k=3$ classifiers in the ensemble have an accuracy of $\accuracy{1} = \accuracy{2} = \accuracy{3} = 70\%$ and also know that they provide in expectation $\information{1} = \information{2} = \information{3} = 0.36$~bit of information, we can improve the ensemble accuracy bounds from 78\%--89\% (with only known accuracies) to 85\%--89\%. This corresponds to Figure~\ref{fig:panel1}, second plot, orange shaded area for $k=3$. A lower information of $0.15$ bit would lead to bounds of 78\%--84\% (blue shaded area).
While the bounds in Theorem~\ref{theorem:IndAccGroupAcc} are tight, we do not know whether the bounds in Theorem~\ref{theorem:IndAccIndInfoGroupAcc} are tight.

One application of Theorem~\ref{theorem:IndAccIndInfoGroupAcc} is to determine how many classifiers are at least necessary to guarantee a target ensemble accuracy of, say, 95\%, see Figure~\ref{fig:panel3}. If the individual classifiers have an accuracy of $70\%$ and a high information of $0.36$~bit (light orange line) an ensemble size of $k=7$ classifiers is required. If their accuracy is the same but their information is lower ($0.26$~bit, light blue line), then $k=13$ classifiers are required to achieve the target ensemble accuracy.

\begin{figure}[t]
\includegraphics[width=.85\linewidth]{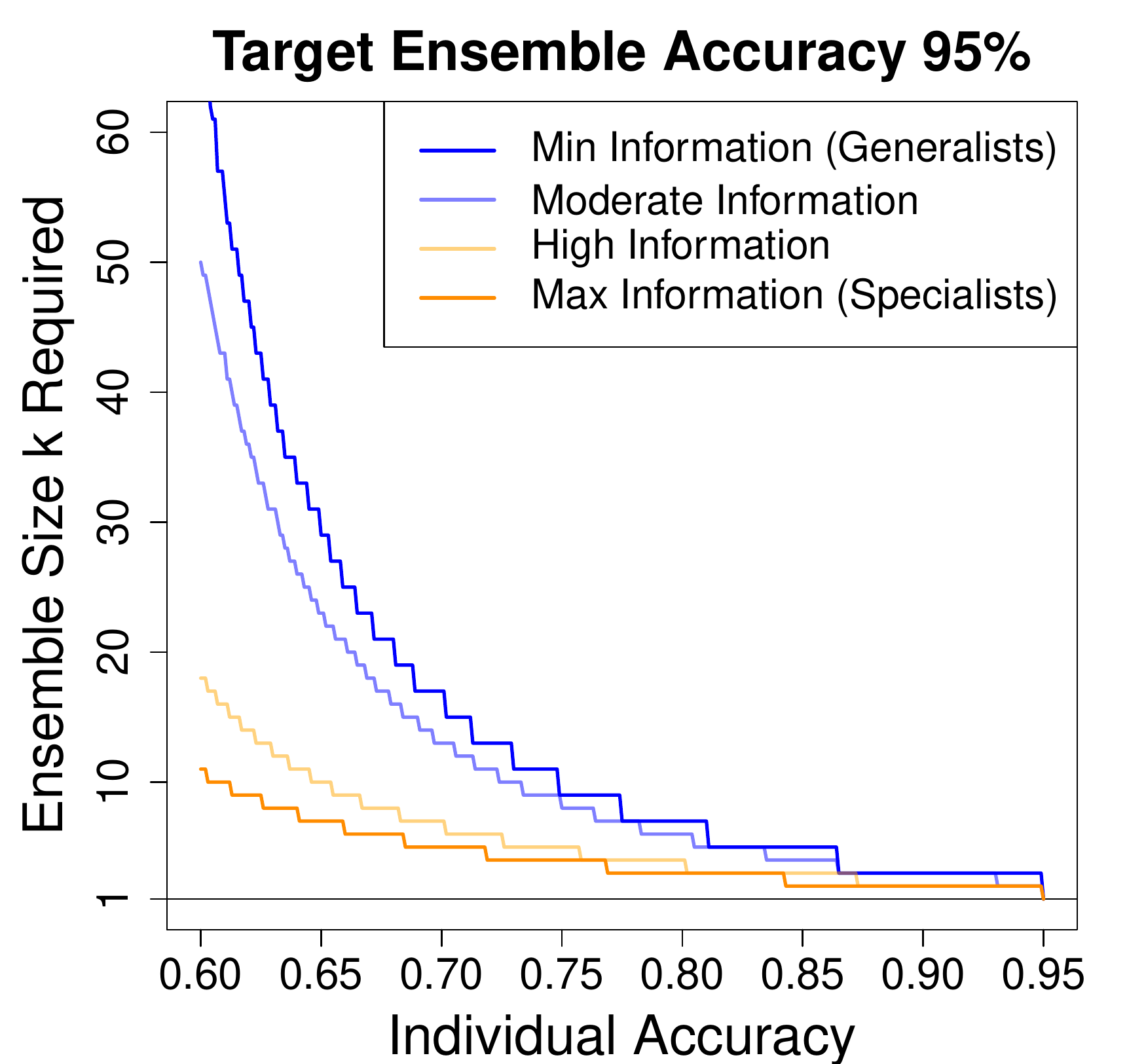}
\centering
\caption{\textbf{Ensemble size determination.} To achieve a target ensemble accuracy of  $\accuracy{f_e} = 95\%$ a certain ensemble size $k$ (y-axis) is required depending on the accuracy of the individual classifiers (y-axis). Across accuracies, we consider 4 different levels of individual classifiers' information: minimal (blue), moderate (50\% of the admissible information range; light blue), high (90\%; light orange) and maximal (orange) information. Low information classifiers (blue) require larger ensembles to reach the target ensemble accuracy than high information classifiers (orange) with the same individual accuracy.}
\label{fig:panel3}
\end{figure}


\subsection{Bounds on the ensemble mutual information}
\label{sec:IndInfGroupX}

Up to now, we have bounded the ensemble accuracy, $\accuracy{f_e}$. Since we introduced the information of an individual classifier as a second measure, we can also consider the bounds for the ensemble information, that is, the mutual information between the true label and the ensemble's output, $\information{f_e}$. The ensemble information behaves much like the ensemble accuracy and is bounded by the same confidence distributions as before.
\begin{proposition}
    \label{prop:IndAccGroupInfo}
    \textbf{(Specialists and generalists bound the ensemble information)}
    Consider $k$ classifiers with individual accuracies $\accuracy{i}$ and confidence distributions $f_i$ ($i \in \{1..k\}$). For each $i$, let 
    $f^\text{generalist}_{\accuracy{i}}$,
    $f^\downarrow_{\accuracy{i},\information{i}}$,
    $f^\uparrow_{\accuracy{i},\information{i}}$
    and $f^\text{specialist}_{\accuracy{i}}$ be as defined above. 
    Now consider the ensemble classifier based on the original classifiers, with ensemble confidence distribution $f_e~=~\bigotimes_{i = 1}^k\,f_i$ as well as ensembles with confidence distributions $f^\text{generalist}_e$, $f^\downarrow_e$, $f^\uparrow_e$, and $f^\text{specialist}_e$ as in Theorem~\ref{theorem:IndAccIndInfoGroupAcc}.
    The information of the ensemble classifier is bounded by
    $$\information{f^\text{generalist}_e} \leq \information{f^\downarrow_e} \leq \information{f_e} \leq \information{f^\uparrow_e} \leq \information{f^\text{specialist}_e}.$$     
\end{proposition}
The proof is in the Supplementary Materials, Section~\ref{sup:boundsOnInformation}. Classifiers with confidence distributions that improve the ensemble accuracy also tend to improve the ensemble information. Finally, we bound the ensemble information for when only the individual classifiers' information is known (but not their accuracies).
\begin{proposition}
    \label{prop:IndInfoGroupInfo}
    \textbf{(Information constrained specialists and generalists bound the ensemble information)}
    Consider $k$ classifiers with individual information $\information{i}$ and confidence distributions $f_i$ ($i \in \{1..k\}$). For each $i$, let the accuracies corresponding to the individual information be $\tilde{\accuracy{i}}~=~H_2^{-1}(1-\information{i})$. Let $f^\text{generalist}_{\overset{\sim}{\accuracy{i}}}$ and $f^\text{specialist}_{\overset{\sim}{\accuracy{i}}}$ as defined above. 
    Now consider the ensemble information based on the original classifiers, with ensemble confidence distribution $f_e~=~\bigotimes_{i = 1}^k\,f_i$ as well as ensembles with confidence distributions $\tilde{f}^\text{generalist}_e~=~\bigotimes_{i = 1}^k\,f^\text{generalist}_{\overset{\sim}{\accuracy{i}}}$ and  $\tilde{f}^\text{specialist}_e~=~\bigotimes_{i = 1}^k\,f^\text{specialist}_{\overset{\sim}{\accuracy{i}}}$ as in Theorem~\ref{theorem:IndAccIndInfoGroupAcc}. The information of the ensemble classifier is bounded by 
    $$\information{\tilde{f}^\text{generalist}_e} \leq \information{f_e} \leq \information{\tilde{f}^\text{specialist}_e}.$$  
\end{proposition}

The proof is in the Supplementary Materials, Section~\ref{sup:boundsOnInformation}. At first sight, this seems to be unsurprising: Again, specialists and generalists attain the upper resp. lower bounds. But specialists have a lower accuracy than generalists to the same information. Consider individual specialists with known information of $\information{i} = 0.4$~bit: In Figure~\ref{fig:proofIdea} (middle) they lie on the dashed line to the left of the crescent (orange dot, with an accuracy of $70\%$). Generalists with the same information lie on the solid curve to the right (with accuracy of around $85\%$, no dot is shown). Having the same information, specialists have a $15\%$-point lower accuracy than generalists but nevertheless produce a better ensemble information! The explanation for this can be found by applying information decomposition \citep{griffith2014quantifying,harder2013bivariate}: Specialists in our setting have a higher proportion of \textit{unique} and a smaller proportion of \textit{redundant} information as compared to generalists and are therefore more effective in ensembles. Thus, Proposition~\ref{prop:IndInfoGroupInfo} demonstrates another desirable property of specialists.

\section{Discussion}
\label{sec:discussion}

In a setting in which individual classifiers output predictions together with confidences (probability elicitation), we have shown that the accuracy of the ensemble depends on the exact confidence distributions. Classifiers that distinguish between high and low confidences perform better than those that always produce moderate confidences. We have provided bounds when (a) only the individual classifiers' accuracies are known and (b) when both, the individual classifiers' accuracies and mutual information, are known. These bounds can be used to determine how many classifiers must be included in an ensemble to guarantee a target accuracy, see Figure~\ref{fig:panel3}.

For our running example, this means that even if we know how often a classifier can predict correctly whether a patient has a disease or not based on a single retina image, we cannot uniquely determine the accuracy of an ensemble of such classifiers. However, we can provide bounds and improve on these bounds when we additionally know the transmitted information of these classifiers. 

Classifiers in an ensemble should ideally be constructed such that they specialize: For a given accuracy, ideal classifiers should sometimes predict with close to absolute certainty even if this comes at the cost of not learning on other parts of the input space. The advantage of specialists comes through despite the specialists not coordinating on which areas of the input space they specialize.

\subsubsection*{Acknowledgements}

This project is supported by the Deutsche Forschungsgemeinschaft (DFG,
German Research Foundation) through the CRC 1233 ``Robust Vision'',
project number 276693517; 
the Institutional Strategy of the University
of T{\"u}bingen (DFG, ZUK 63); and the Cluster of Excellence “Machine
Learning: New Perspectives for Science”, EXC 2064/1, project number
390727645.






\newpage 

\onecolumn
\begin{center}
\textbf{\Large Specialists Outperform Generalists in Ensemble Classification} \\[0.3cm]
\textbf{\Large Supplementary Material}
\end{center}
\setcounter{equation}{0}
\setcounter{figure}{0}
\setcounter{theorem}{0}
\setcounter{table}{0}
\setcounter{section}{0}
\setcounter{subsection}{0}
\setcounter{page}{1}
\makeatletter
\renewcommand{\theequation}{S\arabic{equation}}
\renewcommand{\thefigure}{S\arabic{figure}}
\renewcommand{\bibnumfmt}[1]{[S#1]}
\renewcommand{\citenumfont}[1]{S#1}

\newtheorem{lemma}[theorem]{Lemma}
\newtheorem{corollary}[theorem]{Corollary}
\renewcommand{\thetheorem}{S\arabic{theorem}}


\renewcommand\thesection{\Alph{section}}
\renewcommand\thesubsection{\thesection.\arabic{subsection}}


\section{Ensemble confidence distribution} 
\label{sup:convolution}

In this section, we show how two individual classifiers with confidence distributions $f_1$ and $f_2$ combine into an ensemble classifier with confidence distribution $f_e = f_1 \otimes f_2$. Throughout the Supplementary Materials (except for Section~\ref{sup:rearrangingMutualInformation} to remain consistent with the main text), we will consider only discrete confidence distributions $f(c) = P(C=c)$ with support $\Omega_f = \{c | f(c)>0 \}$ to keep notation simple.


To further simplify notation, we introduce an ad-hoc notation $f^*(C)$ to a given confidence distribution $f(C)$. It redistributes confidence mass from the range of $C\in[0.5,1]$ to $C^*\in[0,1]$ symmetrically: Half the probability mass of $f(C)$ goes to $f^*(C)$ and the other half to $f^*(1-C)$.
\begin{definition}
    \textbf{(Redistributed confidence distribution)}
    \label{def:redist}
    Let $f:[0.5,1]\rightarrow\mathbb{R}$ be a confidence distribution. Then $f^*:[0,1]\rightarrow\mathbb{R}$ is the corresponding redistributed confidence distribution such that
    \begin{align*}
        f^*(c) = 
            \begin{cases} 
                f(1-c)/2 & 0 \leq c < 0.5\\
                f(c) & c = 0.5\\
                f(c)/2 & 0.5 < c \leq 1
            \end{cases} \text{\  ~~~~~~~~~~and~~~~~ \ }
        f(c) = 
            \begin{cases} 
                f^*(c) & c = 0.5\\
                2f^*(c) & 0.5 < c \leq 1
            \end{cases}
    \end{align*}
\end{definition}

Let $g:(0, 1) \times (0,1) \rightarrow (0,1)$ be the function that determines which redistributed confidence, $c^*_2 \in (0,1)$, the second classifier has to produce such that together with the confidence of the first classifier, $c^*_1 \in (0,1)$, a given ensemble confidence, $c^*_e \in (0,1)$, is obtained, $c^*_2 = g(c^*_e, c^*_1)$. Then, the confidence distribution of the ensemble is given by Proposition~\ref{prop:combine}.
\begin{proposition}
    \textbf{(Combining confidence distributions)}
    \label{prop:combine}
    Given are two classifiers with confidence distributions $f_1$ and $f_2$. The ensemble confidence distribution $f_e = f_1 \otimes f_2$ is
    \begin{align*}
    f_e(c_e) = \big(f_1 \otimes f_2\big)(c_e) &= 
        \begin{cases}
            \sum_{c^*_1 \in \Omega_{f^*}\setminus\{0,1\}} \left( f^*_1(c^*_1) \cdot f^*_2(g(c_e, c^*_1)) \cdot \frac{2c^*_1(1-c^*_1)}{c^*_1 +c_e -2c^*_1 c_e  } \right) & c_e = 0.5 \\[0.3cm]
            2 \sum_{c^*_1 \in \Omega_{f^*}\setminus\{0,1\}} \left( f^*_1(c^*_1) \cdot f^*_2(g(c_e, c^*_1)) \cdot \frac{2c^*_1(1-c^*_1)}{c^*_1 +c_e -2c^*_1 c_e  } \right) & 0.5 < c_e < 1 \\[0.3cm]
            f_1(1) + f_2(1) - f_1(1)f_2(1) & c_e = 1
        \end{cases}
    \end{align*}
    where $g(c_e, c^*_1) = \frac{c_e (1-c^*_1) }{-2 c^*_1 c_e + c^*_1 +c_e}$.
\end{proposition}

\begin{proof}

First, we show that $c^*_1\in(0, 1)$ together with $c^*_2 = g(c^*_e, c^*_1)\in(0, 1)$ produces $c^*_e \in (0, 1)$. We rearrange
\begin{align*}
    c^*_e = \frac{1}{1 + \exp\left( - \left( \log\left(\frac{c^*_1}{1- c^*_1}\right) + \log\left(\frac{c^*_2}{1 - c^*_2}\right) \right) \right) }
\end{align*}
so that
\begin{align*}
    c^*_2 = \frac{1}{1 + \exp\left( - \left( \log\left(\frac{c^*_e}{1- c^*_e}\right) - \log\left(\frac{c^*_1}{1 - c^*_1}\right) \right) \right) }
    = \frac{c^*_e (1-c^*_1) }{-2 c^*_1 c^*_e + c^*_1 +c^*_e} = g(c^*_e, c^*_1) \text{.}
\end{align*}
We now show for the two cases, $c_e \in [0.5, 1)$ and $c_e = 1$, that the ensemble confidence distribution returns the probability that the ensemble prediction is correct.

(1) Case $c_e \in (0.5, 1)$: We remove $c^*_1 = 0$ and $c^*_1 = 1$ from the support $\Omega_{f^*}$ because, by convention, they produce $c^*_e = 0$ and $c^*_e = 1$ and therefore $c_e = 1$, which is excluded in this case.
\begin{align*}
    f^*_e(c^*_e) 
        & = 
        \sum_{c^*_1 \in \Omega_{f^*}\setminus\{0,1\}} f^*_1(c^*_1) \cdot
        f^*_2( g(c^*_e, c^*_1)) \cdot \frac{2c^*_1(1-c^*_1)}{-2c^*_1 c_e + c^*_1 +c_e } \\
        & = 
        \sum_{c^*_1 \in \Omega_{f^*}\setminus\{0,1\}} f^*_1(c^*_1) \cdot
        f^*_2( g(c^*_e, c^*_1)) \cdot 2\left( \frac{c^*_e c^*_1(1-c^*_1)}{-2c^*_1 c_e + c^*_1 +c_e } + \frac{(1-c^*_e)c^*_1(1-c^*_1)}{-2c^*_1 c_e + c^*_1 +c_e } \right) \\
        &= 
        \sum_{c^*_1 \in \Omega_{f^*}\setminus\{0,1\}}
        f^*_1(c^*_1) \cdot f^*_2(g(c^*_e, c^*_1)) \cdot 2(c^*_1 g(c^*_e, c^*_1) + (1-c^*_1)(g(1-c^*_e, 1-c^*_1)) \\
        &\overset{(1)}{=} 
        \sum_{c^*_1 \in \Omega_{f^*}\setminus\{0,1\}}
        f^*_1(c^*_1) \cdot f^*_2(g(c^*_e, c^*_1)) \cdot 2(c^*_1 g(c^*_e, c^*_1) + (1-c^*_1)(1-g(c^*_e, c^*_1)) \\
        &=
        \sum_{c^*_1 \in \Omega_{f^*}\setminus\{0,1\}} 
        2 \bigg( c^*_1 \cdot f^*_1(c^*_1) \cdot g(c^*_e, c^*_1) \cdot f^*_2(g(c^*_e, c^*_1)) + 
        (1-c^*_1) \cdot f^*_1(c^*_1) \cdot (1-g(c^*_e, c^*_1)) \cdot f^*_2(g(c^*_e, c^*_1)) \bigg) \\
        &=
        \sum_{c^*_1 \in \Omega_{f^*}\setminus\{0,1\}}  \sum_{y \in \{-1, +1\}} \left( 
        \frac{P(Y=y|C^*_1 = c^*_1)P(C^*_1 = c^*_1)}{P(Y=y)} \frac{P(Y=y|C^*_2 = g(c^*_e, c^*_1))P(C^*_2 = g(c^*_e, c^*_1))}{P(Y=y)} P(Y=y) \right) \\
        &=
        \sum_{c^*_1 \in \Omega_{f^*}\setminus\{0,1\}}  \sum_{y \in \{-1, +1\}} 
        P(C^*_1 = c^*_1|Y=y) P( C^*_2 = g(c^*_e, c^*_1) = g(c^*_e, c^*_1)|Y=y) P(Y=y) \\
        &= 
        \sum_{c^*_1 \in \Omega_{f^*}\setminus\{0,1\}} \sum_{y \in \{-1, +1\}} 
        P(C^*_1 =c^*_1, C^*_2 = g(c^*_e, c^*_1),Y=y)  \\
        &= 
        \sum_{c^*_1 \in \Omega_{f^*}\setminus\{0,1\}} P(C^*_1 =c^*_1, C^*_2 = g(c^*_e, c^*_1)) \\
        &= 
        P(C^*_e = c^*_e) \text{}
\end{align*}
In (1) we use the symmetry, $g(1-c^*_e, 1-c^*_1) = 1-g(c^*_e, c^*_1))$. Plugging these values into Definition~\ref{def:redist} yields the desired result.

(2) Case $c_e = 1$: We solve the edge case using the convention, $c_e = 1 \iff$ $c_1 = 1$ or $c_2 = 1$. Then
\begin{align*}
    f_e(c_e) 
    &= f_1(1) + f_2(1) - f_1(1)f_2(1) \\
    &= P(C_1 = 1) + P(C_2 = 1) - P(C_1 = 1 \land C_2 = 1) \\
    &= P(C_1 = 1 \lor C_2 = 1) \\
    &= P(C_e = 1) \text{.} 
\end{align*}

\end{proof}

The operator $\otimes$ is closed on the space of confidence distributions (probability distributions over $C\in[0.5, 1]$). Its associativity and commutativity follow from associativity and commutativity of addition and multiplication. The neutral element is  $f^\text{generalist}_{0.5}$. Together, this makes the operator $\otimes$ a commutative monoid.

\section{Mutual Information between true label and classifier output}
\label{sup:rearrangingMutualInformation}

Here, we show that the mutual information between true label and a classifier's output is a function that only depends on the classifier's confidence distribution.

\begin{proposition}
    \textbf{(Information is a function of local confidences)}
    Given is a classifier as defined in the main paper that produces the predictions and confidences as output, $O = (\hat{Y}, C)$. The mutual information $I$ between the true label $Y$ and the classifier's output $O$ is
    $$I\left(Y; O\right) = \int_{0.5}^1 f(c) \left(1-H_2(c)\right) ~dc \text{.}$$
\end{proposition}

\begin{proof}
    \begin{align*}
        I\left(Y; O\right)
        = H(Y) - H\left(Y|O\right)
        = H(Y) - H\left(Y| (\hat{Y}, C) \right)
        = \int_{0.5}^1 f(c) \left( H(Y) - H(Y|\hat{Y}, C=c)\right) ~dc
    \end{align*}
    By assumption, $Y$ is binary and equally weighted so that $H(Y) = H_2(0.5) = 1$~bit. To complete the proof, we have to show that $H(Y|\hat{Y}, C=c) = H_2(c)$:
    \begin{align*}
        &H(Y|\hat{Y}, C=c) \\
        &= - \sum_{Y\in\{\pm1\}} \sum_{\hat{Y}\in\{\pm1\}} P(Y,\hat{Y}|C=c) \log_2 P(Y|\hat{Y}, C=c) \\
        &= - \sum_{Y\in\{\pm1\}} \sum_{\hat{Y}\in\{\pm1\}} P(Y)P(\hat{Y}|Y,C=c) \log_2 P(Y|\hat{Y}, C=c) \\
        &= - \sum_{Y\in\{\pm1\}} \sum_{\hat{Y}\in\{\pm1\}} \frac{1}{2}P(\hat{Y}|Y,C=c) \log_2 P(Y|\hat{Y}, C=c) \\
        &= -\frac{1}{2} \left( \underbrace{c \log_2 c}_{Y=\hat{Y} = +1} + \underbrace{c \log_2 c}_{Y=\hat{Y} = -1} + \underbrace{(1-c) \log_2 (1-c)}_{Y=+1\neq\hat{Y} = -1} + \underbrace{(1-c) \log_2 (1-c)}_{Y=-1\neq\hat{Y} = +1} \right) \\
        &= - \left( c \log_2 c + (1-c) \log_2 (1-c) \right) \\
        &= H_2(c) \text{.}\\
    \end{align*}
\end{proof}

\section{Refinement and Jensen's inequality}
\label{sup:refinement}

For all remaining proofs, we will use a partial ordering on the classifiers, called refinement \citep{degroot1983comparison}. In general, we will show here that more refined classifiers have higher scores on so called scoring functions. The remaining sections of the Supplementary Materials then only aim to show that certain functions (the ensemble accuracy, the individual information etc.) are a convex scoring function.

Intuitively, we say a classifier with confidence distribution $f$ is more refined than a classifier with confidence distribution $f'$ if $f'$, instead of producing different confidences $c_1$ and $c_2$, produces an intermediate confidence $c^\text{center} = tc_1 + (1-t)c_2$.
\begin{definition}
    \textbf{(Refinement)}
    \label{def:refinement}
    A classifier with confidence distribution $f$ is more refined than a classifier $f'$, $f\succ f'$, if there exist $c_1$, $c_2 \in[0.5, 1]$, and $\epsilon_1, \epsilon_2\in\mathbb{R}$ such that $0 \leq \epsilon_1 \leq f(c_1)$, $0 \leq \epsilon_2 \leq f(c_2)$ and
    \begin{align*}
        f'(c)  =
        \begin{cases}
          f(c) & c \neq c_1, c \neq c_2, c \neq c^\text{center}\\
          f(c)-\epsilon_1 & c = c_1 \\
          f(c)-\epsilon_2 & c = c_2 \\
          f(c) + \epsilon_1 + \epsilon_2 & c = c^\text{center} \text{.}
        \end{cases} 
    \end{align*}
    where $c^\text{center}$ is the weighted mean, $c^\text{center} = \frac{\epsilon_1 c_1 + \epsilon_2 c_2}{\epsilon_1 +\epsilon_2}$.
    Furthermore, if $f\succ f'$ and $f' \succ f''$ then $f \succ f''$ (transitivity).
\end{definition}
In the main paper, we have considered four particular classifiers: Specialist, more specialized classifier, less specialized classifier and generalist. For a given accuracy, $\accuracy{}$, and information, $\information{}$, these classifiers are in a refinement ordering, $f^\text{specialist}_{\accuracy{}} \succ f^\uparrow_{\accuracy{},\information{}} \succ f^\downarrow_{\accuracy{},\information{}} \succ f^\text{generalist}_{\accuracy{},\information{}}$. For example, it is straight forward to see that a specialist is more refined than a generalist by choosing $c_1 = 0.5$, $c_2 = 1$, $\epsilon_1 = w_{0.5}$ and $\epsilon_2 = w_{1}$ in Definition~\ref{def:refinement} to obtain $f^\text{generalist}_{\accuracy{}}$.

We evaluate confidence distributions, for example, by computing the accuracy or information. These evaluations are based on scoring functions, $\phi:[0.5, 1]\rightarrow \mathbb{R}$, that translate local confidences into values $\phi(c)$.
\begin{definition}
    \textbf{(Score)}
    Given a scoring function $\phi(c):[0.5, 1]\rightarrow \mathbb{R}$, the score of a confidence distribution $f$ is
    \begin{align*}
        \Phi(f) = \sum_{c \in \Omega_c} f(c)\phi(c) ~dc \text{.}
    \end{align*}
\end{definition}
For example, when we choose the scoring function $\phi(c) = 1-H_2(c)$ to evaluate a classifier's confidence distribution $f$, the score is the information, $\Phi(f) = \information{f} = \int_{0.5}^1 f(c)(1-H_2(c)) ~dc$. When we chose the identity scoring function $\phi(c) = c$, the score is the accuracy $\Phi(f) = \accuracy{f} = \int_{0.5}^1 f(c)c ~dc$.

For convex scoring functions, we can apply apply Jensen's inequality,
$$\phi(tc_1 + (1-t)c_2) \leq t\phi(c_1) + (1-t)\phi(c_2) \text{,}$$
to show that less refined confidence distributions (generalists) produce lower scores while more refined confidence distributions (specialists) produce higher scores.

\begin{lemma}
    \textbf{(Jensen's inequality for confidence distributions)}
    \label{lemma:jensens}
    Let $\phi$ be a convex scoring function with score $\Phi_f = \sum_{c\in\Omega_f} f(c)\phi(c)$. If $f$ is more refined than $f'$ then $\accuracy{f} = \accuracy{f'}$ and $\Phi_f \geq \Phi_f'$. %
\end{lemma}

\begin{proof}
    First, $f'$ has the same accuracy as the original $f$:
    \begin{align*}
        \accuracy{f'} &= \accuracy{f} \\
        \sum_{c\in\Omega_f'} f'(c)c &= \sum_{c\in\Omega_f} f(c)c \\
        f'(c^\text{center})c^\text{center}
        &= 
        \epsilon_1c_1 + \epsilon_2c_2 + f(c^\text{center})c^\text{center} \\
        f'(c^\text{center})c^\text{center}
        &= 
        \epsilon_1c_1 + \epsilon_2c_2 + f(c^\text{center})c^\text{center}\\
        f'(c^\text{center})c^\text{center}
        &= 
        (\epsilon_1 + \epsilon_2)\frac{\epsilon_1c_1 + \epsilon_2c_2}{\epsilon_1 + \epsilon_2} + f(c^\text{center})c^\text{center} \\
        f'(c^\text{center})c^\text{center}
        &= 
        (\epsilon_1 + \epsilon_2)c^\text{center} + f(c^\text{center})c^\text{center} \\
        f'(c^\text{center})c^\text{center}
        &= 
        (\epsilon_1 + \epsilon_2 + f(c^\text{center}))c^\text{center} \\
        f'(c^\text{center})c^\text{center}
        &= 
        f'(c^\text{center})c^\text{center} \text{.}
    \end{align*}
    Second, $f'$ has a smaller (or equal) score $\Phi$ as $f$.
    \begin{align*}
        \Phi_f' &\leq \Phi_f\\
        \sum_{c\in\Omega_f'} f'(c)\phi(c)
        &\leq 
        \sum_{c\in\Omega_f} f(c)\phi(c)   
        \\
        f'(c^\text{center})\phi(c^\text{center})
        &\leq 
        \epsilon_1\phi(c_1) + \epsilon_2\phi(c_2) + f(c^\text{center})\phi(c^\text{center})   
        \\
        (f(c^\text{center}) + \epsilon_1 + \epsilon_2)\phi(c^\text{center})
        &\leq 
        (\epsilon_1 + \epsilon_2)\frac{\epsilon_1\phi(c_1) + \epsilon_2\phi(c_2)}{\epsilon_1 + \epsilon_2} + f(c^\text{center})\phi(c^\text{center})  
        \\
        \phi(c^\text{center})
        &\leq 
        \frac{\epsilon_1\phi(c_1) + \epsilon_2\phi(c_2)}{\epsilon_1 + \epsilon_2}  
        \\
        \phi\left(\frac{\epsilon_1c_1 + \epsilon_2c_2}{\epsilon_1 + \epsilon_2}\right)
        &\leq
        \frac{\epsilon_1\phi(c_1) + \epsilon_1f(c_2)\phi(c_2)}{\epsilon_1 + \epsilon_2} 
        \\
        \phi\left(tc_1 + (1-t)c_2\right)
        &\leq
        t\phi(c_1) + (1-t)\phi(c_2)
    \end{align*}
    The last inequality holds due to Jensen's inequality for convex $\phi$.
\end{proof}

The immediate consequence is that generalists produce the lowest score and specialists produce the highest score for any convex scoring function.
\begin{corollary}
    \textbf{(Generalist and specialist produce minimal and maximal value of convex scoring functions)}
    \label{corollary:genAndSpec}
    Let $f(c)$ be a confidence distribution with fixed accuracy $\accuracy{f}$ and $\phi(c)$ be a convex scoring function with score $\Phi_f$. The score is minimized by the generalist and maximized by the specialist.
    \begin{align*}
        \min_f \Phi_f = \Phi_f^\text{generalist} \text{~~~~ and ~~~~} \max_f \Phi_f = \Phi_f^\text{specialist}
    \end{align*}
\end{corollary}

\begin{proof}
    For any confidence distribution $f$ with accuracy $\accuracy{} = \accuracy{f}$ and probability mass at different confidences, $c_1 \neq c_2$ with $f(c_1)>0$ and $f(c_2)>0$, $\exists: f': f'\prec f$ so that $\Phi(f') \leq \Phi(f)$. By induction, $\min_f \Phi_f$ is obtained by the least refined confidence distribution, $f^\text{generalist}_{\accuracy{}}$.
    
    For any confidence distribution $f'$ with accuracy $\accuracy{} = \accuracy{f'}$ and probability mass at confidence $0.5<c^\text{center}<1$ with $f(c^\text{center})>0$, $\exists f: f'\prec f$ such that $\Phi(f')~\leq~\Phi(f)$. By induction, $\max_f \Phi_f$ is obtained by the most refined confidence distribution, $f^\text{specialist}_{\accuracy{}}$.
\end{proof}


\section{Ensemble accuracy bounds for given individual accuracies}

We now prove the bounds on the ensemble accuracy. Confidence distributions $f$ in the following proofs will be discrete probability distributions with support $\Omega_f = \{c|f(c) > 0\}$ to simplify notation. The continuous case follows by generalizing Jensen's inequality to the continuous functions. The tricky part of the proofs is not handling the continuous case; the tricky part is to show convexity of several functions so that we can apply Corollary~\ref{corollary:genAndSpec}.

\label{sup:IndAccEnsembleAcc}

\setcounterref{theorem}{theorem:IndAccGroupAcc}
\addtocounter{theorem}{-1}
\renewcommand{\thetheorem}{\arabic{theorem}}

\begin{theorem}
    \textbf{(Specialists and generalists bound the ensemble accuracy)}
    Consider $k$ classifiers with individual accuracies $\accuracy{i}$ and confidence distributions $f_i$ ($i \in \{1..k\}$). For each $i$, let 
    $f^\text{generalist}_{\accuracy{i}}$ and $f^\text{specialist}_{\accuracy{i}}$ be a generalist resp. specialist classifier that has the same accuracy as classifier $i$. 
    Now consider the ensemble classifier based on the original classifiers with ensemble confidence distribution $f_e~=~\bigotimes_{i = 1}^k\,f_i$ according to $l$CWMV as well as the ensemble of generalists and ensemble of specialists with ensemble confidence distributions $f^\text{generalist}_e~=~\bigotimes_{i = 1}^k\,f^\text{generalist}_{\accuracy{i}}$ and $f^\text{specialist}_e~=~\bigotimes_{i = 1}^k\,f^\text{specialist}_{\accuracy{i}}$. Then the accuracy of the original ensemble is lower and upper bounded by the accuracies of the generalist and specialist ensembles: 
    $$\accuracy{f^\text{generalist}_e} \leq \accuracy{f_e} \leq \accuracy{f^\text{specialist}_e}~.$$
\end{theorem}

\begin{proof}

The ensemble accuracy for two classifiers is
\begin{align*}
    \accuracy{f_e} &= \sum_{c\in\Omega_{f_e}} f_e(c)c \\
            &= \sum_{c_1\in\Omega_{f_1}} \sum_{c_2\in\Omega_{f_2}} f_1(c_1)f_2(c_2) P(\hat{y}_e\text{ correct}|c_1, c_2)  \text{.}\\
\end{align*}
Expanding $P(\hat{y}_e\text{ correct}|c_1, c_2)$ yields
\begin{align*}
    P(\hat{y}_e\text{ correct}|c_1, c_2) 
    &= P(\hat{y}_e\text{ correct},\hat{y}_1=\hat{y}_2|c_1, c_2) + P(\hat{y}_e\text{ correct},\hat{y}_1 \neq \hat{y}_2| c_1, c_2) \\
    &=c_1c_2 + \max\{c_1(1-c_2),(1-c_1)c_2\} \\
    &=\max\{c_1,c_2\} \text{}
\end{align*}
In continuation, the ensemble accuracy is
\begin{align*}
    \accuracy{f_e} = \sum_{c_1\in\Omega_{f_1}} \sum_{c_2\in\Omega_{f_2}} f_1(c_1)f_2(c_2) \max\{c_1, c_2\}  \text{.}\\
\end{align*}
The function $\phi(c_1) = \sum_{c_2\in\Omega_{f_2}} f_2(c_2)\max\{c_1, c_2\}$ is convex in $c_1$ for any $c_2$ because $\max$ is convex. The sum (over $c_2$) of convex functions remains convex. Thus, $\phi(c_1)$ is a convex scoring function for $f_1$. Corollary~\ref{corollary:genAndSpec} yields that generalists vs. specialists minimize vs. maximize the score, proving the desired statement for two classifiers. By induction, we obtain the desired result.

\end{proof}


\section{Ensemble accuracy bounds for given individual accuracies and information}
\label{sup:IndAccBoundsIndInfo}

Since not only the ensemble accuracy is convex in individual confidences but also the mutual information, generalists and specialists yield minimal and maximal values in both cases.

\begin{proposition}
    \textbf{(Specialists and generalists bounds the individual information)}
    A classifier with confidence distribution $f$ and accuracy $\accuracy{}$ has an information between
    \begin{equation*}
        \information{f^\text{generalist}_{\accuracy{}}} \leq \information{f} \leq \information{f^\text{specialist}_{\accuracy{}}}.
    \end{equation*}
\end{proposition}

\begin{proof}
The information is
\begin{align*}
    \information{f} &= I(Y; (\hat{Y}, C)) = \sum_{c\in\Omega_{f}} f(c) \left( 1-H_2(c) \right) \text{.} \\
\end{align*}
We derive the second derivative.
\begin{align*}
    1-H_2(c)~~&= 1 - \left( c\log_2\left(\frac{1}{c}\right) + (1-c)\log_2\left(\frac{1}{1-c}\right) \right) \\
    \frac{\mathrm{d}}{\mathrm{d}c} \left( 1-H_2(c) \right) &= \log_2\left( \frac{c}{1-c} \right) \\
    \frac{\mathrm{d}^2}{\mathrm{d}c^2} \left( 1-H_2(c) \right) &= \frac{1}{\log_e(2)c(1-c)}
\end{align*}
The second derivative is strictly larger than 0 in $c\in[0.5, 1)$ so that $\phi(c) = 1-H_2(c)$ is convex. Corollary~\ref{corollary:genAndSpec} yields the desired statement.
\end{proof}

\section{Ensemble accuracy bounds based on both, individual accuracies and information}
\label{sup:IndAccIndInfoEnsembleAcc}

Individual accuracy and information of a classifier still do not determine its confidence distribution. For the proof of Theorem~\ref{theorem:IndAccIndInfoGroupAcc}, we follow this strategy: For each classifier $f$ with given accuracy and information, we construct a more refined classifier $f^\nearrow$. Since there are multiple $f$ that satisfy the two constraints, there are multiple $f^\nearrow$. We then construct one unique $f^\uparrow$ (more specialized classifier) that is more refined than any $f^\nearrow$. By transitivity, $f_i \prec f^\nearrow_{\accuracy{f_i},\information{f_i}} \prec f^\uparrow_{\accuracy{f_i},\information{f_i}}$. Therefore, $f^\uparrow$ improves the ensemble accuracy.

Analogously, $f$ is less refined than $f^\searrow$ and the unique $f^\downarrow$ (less specialized classifier) is even less refined than $f^\searrow$. Thus $f^\downarrow_{\accuracy{f_i},\information{f_i}} \prec f^\searrow_{\accuracy{f_i},\information{f_i}} \prec f_i$ and $f^\downarrow_{\accuracy{f_i},\information{f_i}}$ makes the ensemble accuracy worse.

We will consider the left conditional confidence distribution and the right conditional confidence distribution. By that we mean the confidence distribution $f(C)$ conditioned on $C<\accuracy{f}$ resp. $C\geq\accuracy{f}$. The probabilities to obtain a below or above average confidence are $p^\text{left} = P(C<\accuracy{f})$ resp. $p^\text{right} = P(C\geq\accuracy{f})$.
The left and right conditional accuracies are
\begin{align*}
    \accuracy{f}^\text{left} = \sum_{c\in\Omega_f, c<\accuracy{f}} \frac{f(c)}{p^\text{left}} \cdot c
    \quad\quad\quad\quad\text{and}\quad\quad\quad\quad
    \accuracy{f}^\text{right} = \sum_{c\in\Omega_f, c\geq\accuracy{f}} \frac{f(c)}{p^\text{left}} \cdot c \text{.}
\end{align*}
These are the accuracies of the classifier $f$ when only counting below average (left) or above average (right) confidences. Analogously, the left and right conditional information are
\begin{align*}
    \information{f}^\text{left} = \sum_{c\in\Omega_f, c<\accuracy{f}} \frac{f(c)}{p^\text{left}}\cdot (1-H_2(c))
    \quad\quad\quad\quad\text{and}\quad\quad\quad\quad
    \information{f}^\text{right} = \sum_{c\in\Omega_f, c\geq\accuracy{f}} \frac{f(c)}{p^\text{right}}\cdot (1-H_2(c)) \text{.}
\end{align*}

\begin{figure}[!ht]
\vspace{1in}
\centering
\includegraphics[width=.6\linewidth]{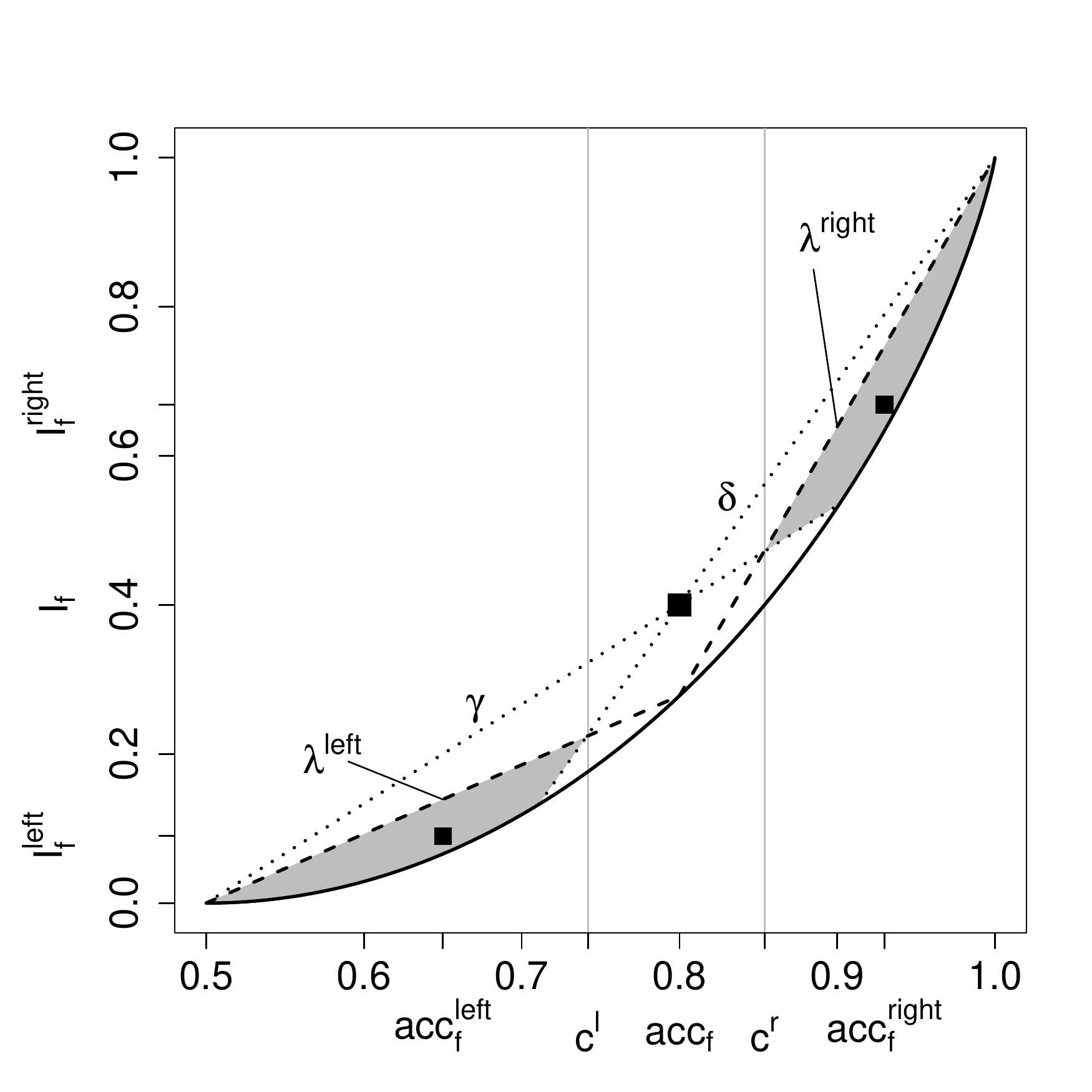}
\caption{\textbf{Proof idea of Theorem~\ref{theorem:IndAccIndInfoGroupAcc}} An individual classifier's confidence distribution $f$ is characterized by a point in the accuracy-information plot. It has accuracy $\accuracy{f}$ (x-axis coordinate) and information $\information{f}$ (y-axis coordinate). Consider the left and right conditional distributions that conditioned on $C<\accuracy{f}$ and $C\geq \accuracy{f}$, respectively. These confidence distributions have the (accuracy, information)-pairs: ($\accuracy{}^\text{left}$, $\information{}^\text{left}$), and ($\accuracy{}^\text{right}$, $\information{}^\text{right}$), which must lie in the shaded areas.}
\label{fig:leftRightVisualization}
\end{figure}

Now, we can define the more and the less refined classifier we need for the proof. The idea for the more refined classifier is to split confidences below average ($C<\accuracy{f}$) to $C=0.5$ and $C=\accuracy{f}$ and to split confidences above average ($C\geq\accuracy{f}$) to $C=\accuracy{f}$ and $C=1$. We end up with a classifier that outputs only three possible confidences: $C=0.5$, $C=\accuracy{f}$ and $C=1$. The probability masses for these cases depend on the original classifier's confidence distribution $f$.
\renewcommand{\thetheorem}{S\arabic{theorem}}
\begin{definition}
    \label{def:moreRefined}
    \textbf{(More refined classifier)}
    A binary black-box classifier to the accuracy $\accuracy{}$ and information $\information{}$ is called more refined if its confidence distribution is given by
    \begin{align*}
        f^\nearrow_{\accuracy{f},\information{f}} = w^\nearrow_{0.5}\delta_{\accuracy{}} + w^\nearrow_{\accuracy{f}}\delta_{\accuracy{f}} + w^\nearrow_{1}\delta_{1} \text{}
    \end{align*}
    with constants 
    $w^\nearrow_{0.5} = \frac{\accuracy{f} - \accuracy{f}^\text{left}}{\accuracy{f} - 0.5}$, 
    $w^\nearrow_{\accuracy{f}} = \frac{\accuracy{f}^\text{left} - 0.5}{\accuracy{f} - 0.5} + 
    \frac{1 - \accuracy{f}^\text{right}}{1 - \accuracy{f}}$, and
    $w^\nearrow_{1} = \frac{\accuracy{f}^\text{right} - \accuracy{f}}{1-\accuracy{f}}$.
\end{definition}

In analogy, the less refined classifier does not split the left and right conditional confidence distributions but fully merges them into $C = \accuracy{f}^\text{left}$ and $C = \accuracy{f}^\text{right}$.
\begin{definition}
    \label{def:lessRefined}
    \textbf{(Less refined classifier)}
    A binary black-box classifier to the accuracy $\accuracy{}$ and information $\information{}$ is called less refined if its confidence distribution is given by
    \begin{align*}
        f^\searrow_{\accuracy{},\information{}} =
        w^\searrow_{\accuracy{f}^\text{left}}\delta_{\accuracy{f}^\text{left}} +
        w^\searrow_{\accuracy{f}^\text{right}}\delta_{\accuracy{f}^\text{right}} \text{,}
    \end{align*}
    with constants $w^\searrow_{\accuracy{f}^\text{left}} = \frac{\accuracy{f}^\text{right} - \accuracy{f}}{\accuracy{f}^\text{right} - \accuracy{f}^\text{left}}$, and
    $w^\searrow_{\accuracy{f}^\text{left}} = \frac{\accuracy{f} - \accuracy{f}^\text{left}}{\accuracy{f}^\text{right} - \accuracy{f}^\text{left}}$.
\end{definition}
These two classifiers, even though similar, are different to the more specialized resp. less specialized classifier. We exploit that they are in a refinement relation to them, which allows us to prove Theorem~\ref{theorem:IndAccIndInfoGroupAcc}.

\setcounterref{theorem}{theorem:IndAccIndInfoGroupAcc}
\addtocounter{theorem}{-1}
\renewcommand{\thetheorem}{\arabic{theorem}}

\begin{theorem}
    \textbf{(More and less specialized classifiers bound the ensemble accuracy)}
    Consider $k$ classifiers with individual accuracies $\accuracy{i}$, individual information $\information{i}$ and confidence distributions $f_i$ ($i \in \{1..k\}$). For each $i$, let 
    $f^\downarrow_{\accuracy{i},\information{i}}$ and $f^\uparrow_{\accuracy{i},\information{i}}$ be the less resp. more specialized classifier constructed to the accuracy and information of classifier $i$. 
    Now consider the ensemble classifier based on the original classifiers with ensemble confidence distribution $f_e~=~\bigotimes_{i = 1}^k\,f_i$ according to $l$CWMV as well as the ensemble of less and more specialized classifiers with ensemble confidence distributions $f^\downarrow_e~=~\bigotimes_{i = 1}^k\,f^\downarrow_{\accuracy{i},\information{i}}$ and $f^\uparrow_e~=~\bigotimes_{i = 1}^k\,f^\uparrow_{\accuracy{i},\information{i}}$. Then the accuracy of the original ensemble is lower and upper bounded by the accuracies of the less and more specialized ensembles: 
    $$\accuracy{f^\text{generalist}_e} \leq \accuracy{f^\downarrow_e} \leq \accuracy{f_e} \leq \accuracy{f^\uparrow_e} \leq \accuracy{f^\text{specialist}_e}.$$    
\end{theorem}

\begin{proof}

First, we show the upper bound in (1) and then the lower bound in (2). Our strategy will be to prove that there is a refinement ordering, $\forall i\in\{1..k\}: f^\downarrow_{\accuracy{f_i},\information{f_i}} \prec f^\searrow_{\accuracy{f_i},\information{f_i}} \prec f_i \prec f^\nearrow_{\accuracy{f_i},\information{f_i}} \prec f^\uparrow_{\accuracy{f_i},\information{f_i}}$. Because more refined classifiers also produce higher ensemble accuracies, Lemma~\ref{lemma:jensens} produces the desired result.

(1) By construction, $f_i \prec f^\nearrow_{\accuracy{f_i},\information{f_i}} $. Let the information gain be $g_{f_i} = \information{f^\nearrow_{\accuracy{f_i},\information{f_i}}} - \information{f_i}$. Let $f^*$ be the confidence distribution that produces the maximal gain, $f^* = \argmax_f g_f$ s.t. $\accuracy{f_i} = \accuracy{i}$ and $\information{f_i} = \information{i}$. Let the maximal gain be $g = g_{f^*}$ (this is the constant in Definition~\ref{def:moreSpecialized}). It remains to show that $\forall f: f^\nearrow_{\accuracy{f},\information{f}} \prec f^\uparrow_{\accuracy{f},\information{f}}$.

For all $f$, the more refined classifier distribution $f^\nearrow_{\accuracy{f},\information{f}}$ is defined by constants $w^\nearrow_{0.5}~=~ \frac{2(1-\accuracy{f})(\information{f}+g_f - 1 + H_2(\accuracy{f}))}{2\accuracy{f} - 2 + H_2(\accuracy{f})}$, 
$w^\nearrow_{\accuracy{f}}~=~\frac{2\accuracy{f} - 1 - (\information{f}+g_f)}{2\accuracy{f} - 2 + H_2(\accuracy{f})}$ and $w^\nearrow_{1}~=~\frac{2(\accuracy{f}-0.5)(\information{f}+g_f - 1 + H_2(\accuracy{f}))}{2\accuracy{f} - 2 + H_2(\accuracy{f})}$. The more specialized classifier distribution is $f^\nearrow_{\accuracy{f},\information{f}}$ with constants $w_{0.5}$, $w_{\accuracy{f}}$ and $w_1$ as in Definition~\ref{def:moreSpecialized}. To prove $f^\nearrow_{\accuracy{f},\information{f}} \prec f^\uparrow_{\accuracy{f},\information{f}}$, we apply Lemma~\ref{lemma:jensens} with $c_1 = 0.5$, $c_2 = 1$, $\epsilon_1 = w_{0.5} - w^\nearrow_{0.5}$, $\epsilon_2 = w_{1} - w^\nearrow_{1}$. It remains to show that these constants transform the more specialized classifier into the more refined classifier: $\epsilon_1 + \epsilon_2 + w_{\accuracy{f}} = w^\nearrow_{\accuracy{f}}$ and that $c^\text{center} = \accuracy{f}$:

\begin{align*}
    \epsilon_1 + \epsilon_2 + w_{\accuracy{f}} &= w_{0.5} - w^\nearrow_{0.5} + w_{1} - w^\nearrow_{1} + w_{\accuracy{f}} \\
    &= 1 - w^\nearrow_{0.5} - w^\nearrow_{1} \\
    &= w^\nearrow_{\accuracy{f}}\\
    c^\text{center} &= \frac{\epsilon_1 c_1 + \epsilon_2 c_2}{\epsilon_1 + \epsilon_2} \\
    &= \frac{2(1-\accuracy{f})\cdot 0.5 + 2(\accuracy{f}-0.5) \cdot 1}{2(1-\accuracy{f}) + 2(\accuracy{f}-0.5)} \\
    &= 2(1-\accuracy{f})\cdot 0.5 + 2(\accuracy{f}-0.5) \\
    &= \accuracy{f}
\end{align*}

Taken together, $f_i \prec f^\nearrow_{\accuracy{f_i},\information{f_i}} \prec f^\uparrow_{\accuracy{f_i},\information{f_i}}$. By Lemma~\ref{lemma:jensens} follows the desired result for (1).

(2) By construction, $f^\searrow_{\accuracy{f_i},\information{f_i}} \prec f_i$. The left conditional accuracy and information pair, ($\accuracy{f}^\text{left}$, $\information{f}^\text{left}$), must lie below the line $\lambda^\text{left}$ that runs through points $(0.5, 0)$ and $(\pi, 1-H_2(\pi))$ because of Proposition~\ref{prop:IndAccIndInfo}. See Figure~\ref{fig:leftRightVisualization} for a visualization. The right conditional information pair, ($\accuracy{f}^\text{right}$, $\information{f}^\text{right}$), must lie below the line $\lambda^\text{right}$ that runs through $(\pi, 1-H_2(\pi))$ and $(1,1)$. In consequence, the left conditional pair must lie above the line $\delta$ running through $(\accuracy{f}, \information{f})$ and $(1,1)$: Assuming for the sake of contradiction that this was not the case entails that the right conditional pair would have to lie above $\lambda^\text{right}$. Analogously, the right conditional pair must lie below line $\gamma$ that runs through $(0.5, 0)$ and $(\accuracy{f}, \information{f})$. Thus, the innermost left and right conditional accuracies are the intersections of these lines, $c^\text{l}$ resp. $c^\text{r}$, see Definition~\ref{def:lessSpecialized}.

For all $f$, the less refined classifier's confidence distribution is $f^\searrow_{\accuracy{f},\information{f}}$ with left and right conditional accuracies $\accuracy{f}^\text{left}$ and $\accuracy{f}^\text{right}$. The less specialized classifier's confidence distribution is $f^\downarrow_{\accuracy{f},\information{f}}$ with left and right conditional accuracies $c^\text{l}$ and $c^\text{r}$. To prove $f^\downarrow_{\accuracy{f},\information{f}} \prec f^\searrow_{\accuracy{f},\information{f}}$, we apply Lemma~\ref{lemma:jensens} with constants $c_1 = \accuracy{f}^\text{left}$, $c_2 = \accuracy{f}^\text{right}$ twice: (a) with
$\epsilon_1' = \frac{(\accuracy{f}^\text{right}-c^\text{l})(c^\text{r}-\accuracy{f})}{(\accuracy{f}^\text{right} - \accuracy{f}^\text{left})(c^\text{right} - c^\text{left})}$ and
$\epsilon_2' = \frac{(\accuracy{f}^\text{left}-c^\text{l})(c^\text{r} - \accuracy{f})}{(\accuracy{f}^\text{right} - \accuracy{f}^\text{left})(c^\text{right} - c^\text{left})}$; and (b) with 
$\epsilon_1'' = \frac{(c^\text{r} - \accuracy{f}^\text{left})(\accuracy{f} - c^\text{l})}{(\accuracy{f}^\text{right} - \accuracy{f}^\text{left})(c^\text{right} - c^\text{left})}$ and 
$\epsilon_2'' = \frac{(c^\text{r} - \accuracy{f}^\text{left})(\accuracy{f} - c^\text{l})}{(\accuracy{f}^\text{right} - \accuracy{f}^\text{left})(c^\text{right} - c^\text{left})}$.
It remains to show that with these constants transform the less refined classifier into the less specialized classifier: 
$\epsilon_1' + \epsilon_1'' = w_{\accuracy{f}^\text{left}}$,
$\epsilon_2' + \epsilon_2'' = w_{\accuracy{f}^\text{right}}$, 
$\epsilon_1' + \epsilon_2' = w_{c^\text{l}}$ and 
$\epsilon_1'' + \epsilon_2'' = w_{c^\text{r}}$.
${c^\text{center}}' = c^\text{l}$, and
${c^\text{center}}'' = c^\text{r}$.

The two deviations on the left hand side deviations add up to the total weight of the left side of the less refined classifier.
\begin{align*}
    \epsilon_1' + \epsilon_1'' 
    &= 
    \frac{(\accuracy{f}^\text{right}-c^\text{l})(c^\text{r}-\accuracy{f})}{(\accuracy{f}^\text{right} - \accuracy{f}^\text{left})(c^\text{r} - c^\text{l})}
    {(\accuracy{f}^\text{right} - \accuracy{f}^\text{left})(c^\text{r} - c^\text{l}) } \\
    & = \frac{(\accuracy{f}^\text{right} - \accuracy{f})(c^\text{r} - c^\text{l})}{(\accuracy{f}^\text{right} - \accuracy{f}^\text{left})(c^\text{r} - c^\text{l})} \\
    & = \frac{\accuracy{f}^\text{right} - \accuracy{f}}
    {\accuracy{f}^\text{right} - \accuracy{f}^\text{left}} \\
    & = w_{\accuracy{f}^\text{left}}\\[0.4cm]
\end{align*}
The deviations produce the left weight of the less specialized classifier.
\begin{align*}
    \epsilon_1' + \epsilon_2' 
    &= \frac{(\accuracy{f}^\text{right}-c^\text{l})(c^\text{r}-\accuracy{f})}{(\accuracy{f}^\text{right} - \accuracy{f}^\text{left})(c^\text{right} - c^\text{left})} + \frac{(\accuracy{f}^\text{left}-c^\text{l})(c^\text{r} - \accuracy{f})}{(\accuracy{f}^\text{right} - \accuracy{f}^\text{left})(c^\text{right} - c^\text{left})}\\
    &= \frac{(\accuracy{f}^\text{right}-\accuracy{f}^\text{left})(c^\text{r}-\accuracy{f})}{(\accuracy{f}^\text{right} - \accuracy{f}^\text{left})(c^\text{right} - c^\text{left})}\\
    &= \frac{c^\text{r}-\accuracy{f}}{c^\text{right} - c^\text{left}}\\
    &= w_{c^\text{l}} \\[0.4cm]
\end{align*}
The left hand side accuracy is kept constant.
\begin{align*}
    {c^{\text{center}}}'
    & = \frac{\epsilon_1' c_1 + \epsilon_2' c_2}{\epsilon_1' +\epsilon_2'} \\
    & = \frac{\frac{(\accuracy{f}^\text{right}-c^\text{l})(c^\text{r}-\accuracy{f})}{(\accuracy{f}^\text{right} - \accuracy{f}^\text{left})(c^\text{r} - c^\text{l})} \accuracy{f}^\text{left} 
    + \frac{(c^\text{l} - \accuracy{f}^\text{left})(c^\text{r} - \accuracy{f})}{(\accuracy{f}^\text{right} - \accuracy{f}^\text{left})(c^\text{r} - c^\text{l})} \accuracy{f}^\text{right}}
    {\frac{(\accuracy{f}^\text{right}-c^\text{l})(c^\text{r}-\accuracy{f})}{(\accuracy{f}^\text{right} - \accuracy{f}^\text{left})(c^\text{r} - c^\text{l})}
    + \frac{(c^\text{l} - \accuracy{f}^\text{left})(c^\text{r} - \accuracy{f})}{(\accuracy{f}^\text{right} - \accuracy{f}^\text{left})(c^\text{r} - c^\text{l})}} \\
    & = \frac{\accuracy{f}^\text{right} \accuracy{f}^\text{left} - c^\text{l} \accuracy{f}^\text{left} + c^\text{l}\accuracy{f}^\text{right} - \accuracy{f}^\text{left}\accuracy{f}^\text{right}   }{\accuracy{f}^\text{right} - \accuracy{f}^\text{left}} \\
    & = \frac{ c^\text{l} (\accuracy{f}^\text{right} - \accuracy{f}^\text{left}) }{\accuracy{f}^\text{right} - \accuracy{f}^\text{left}} \\
    & = c^\text{l}
\end{align*}

The rest follows analogously. Taken together, $f^\downarrow_{\accuracy{f_i},\information{f_i}} \prec f^\searrow_{\accuracy{f_i},\information{f_i}} \prec f_i$. By Lemma~\ref{lemma:jensens} follows the desired result (2).

\end{proof}

\section{Ensemble mutual information bounds}
\label{sup:boundsOnInformation}

Here, we proof the bounds on the ensemble information. To avoid clutter, we now use the natural logarithm $\log_e$ instead of $\log_2$ as in the main text and drop the subscript. Results are transferable because convexity does not change with the base of the logarithm.

\begin{proposition}
    \textbf{(Specialists and generalists bound the ensemble information)}
    Consider $k$ classifiers with individual accuracies $\accuracy{i}$ and confidence distributions $f_i$ ($i \in \{1..k\}$). For each $i$, let 
    $f^\text{generalist}_{\accuracy{i}}$,
    $f^\downarrow_{\accuracy{i},\information{i}}$,
    $f^\uparrow_{\accuracy{i},\information{i}}$
    and $f^\text{specialist}_{\accuracy{i}}$ be as defined above. 
    Now consider the ensemble classifier based on the original classifiers, with ensemble confidence distribution $f_e~=~\bigotimes_{i = 1}^k\,f_i$ as well as ensembles with confidence distributions $f^\text{generalist}_e$, $f^\downarrow_e$, $f^\uparrow_e$, and $f^\text{specialist}_e$ as in Theorem~\ref{theorem:IndAccIndInfoGroupAcc}.
    The information of the ensemble classifier is bounded by
    $$\information{f^\text{generalist}_e} \leq \information{f^\downarrow_e} \leq \information{f_e} \leq \information{f^\uparrow_e} \leq \information{f^\text{specialist}_e}.$$     
\end{proposition}

\begin{proof}
    The ensemble information is
    \begin{align*}
        \information{f_e} 
                &= \sum_{c\in\Omega_{f_e}} f_e(c)\left(H_2(0.5) - H_2(c)\right) \\
                &= \sum_{c_1\in\Omega_{f_1}} \sum_{c_2\in\Omega_{f_2}} f_1(c_1)f_2(c_2) \Big( 
                P(\hat{y_1} = \hat{y_2}|c_1, c_2) (H_2(0.5) - H_2(P(\hat{y}_e\text{ correct}|c_1, c_2, \hat{y_1} = \hat{y_2})))
                \\ &~~~~~~~~~~~~~~~~~~~~~~~~~~~~~~~~~~~~~~~~~~~~~~~~ + P(\hat{y_1} \neq \hat{y_2}|c_1, c_2) (H_2(0.5) - H_2(P(\hat{y}_e\text{ correct}|c_1, c_2, \hat{y_1} \neq \hat{y_2}))) \Big)  \\
    \end{align*}
    
    We will use the following notation to simplify the term in the big brackets. \\[0.5cm]
    \begin{tabular}{p{1cm}lp{1cm}l}
        & $\nu^\text{agree}_{c_1, c_2} = c_1 \cdot c_2 +  (1-c_1) \cdot (1 - c_2)$
        & & (conditional probability to agree) \\[0.2cm]
        & $\nu^\text{disagree}_{c_1, c_2} = c_1 \cdot (1 - c_2) +  c_1 \cdot (1 - c_2)$
        & & (conditional probability to disagree) \\[0.2cm]
        & $\eta^\text{agree}_{c_1, c_2} = \frac{c_1 \cdot c_2}{\nu^\text{agree}}$
        & & (conditional confidence upon agreement) \\[0.2cm]
        & $\eta^\text{disagree}_{c_1, c_2} = \frac{\max\{ c_1 \cdot (1 - c_2),  c_1 \cdot (1 - c_2) \}}{\nu^\text{disagree}_{c_1, c_2}}$
        & & (conditional confidence upon disagreement)
    \end{tabular} \\[0.5cm]
    With this notation, the term in the big brackets is
    \begin{align*}
        \phi_{c_2}(c_1) = \Big( 
                \nu^\text{agree}_{c_1, c_2} (H_2(0.5) - H_2(\eta^\text{agree}_{c_1, c_2}))
                + \nu^\text{disagree}_{c_1, c_2} (H_2(0.5) - H_2(\eta^\text{disagree}_{c_1, c_2})) \Big)
    \end{align*}
    and we will show that it is convex in $c_1$. We do so by showing that its second derivative is non-negative. First, we rearrange.
    \begin{align*}
        \phi_{c_2}(c_1)
                &= \Big( 
                \nu^\text{agree}_{c_1, c_2} (H_2(0.5) - H_2(\eta^\text{agree}_{c_1, c_2}))
                + \nu^\text{disagree}_{c_1, c_2} (H_2(0.5) - H_2(\eta^\text{disagree}_{c_1, c_2})) \Big)  \\
                &= H_2(0.5) - \Big( 
                \nu^\text{agree}_{c_1, c_2} H_2(\eta^\text{agree}_{c_1, c_2}
                + \nu^\text{disagree}_{c_1, c_2} H_2(\eta^\text{disagree}_{c_1, c_2}) \Big)  \\
                &= H_2(0.5) - \Big( 
                \nu^\text{agree}_{c_1, c_2} H_2\left(\frac{c_1c_2}{\nu^\text{agree}_{c_1, c_2}}\right)
                + \nu^\text{disagree}_{c_1, c_2} H_2\left(\frac{c_1(1-c_2)}{\nu^\text{disagree}_{c_1, c_2}}\right) \Big)  \\
                &= H_2(0.5) - \Big( 
                \nu^\text{agree}_{c_1, c_2} H_2\left(\frac{c_1c_2}{\nu^\text{agree}_{c_1, c_2}}\right)
                + \nu^\text{disagree}_{c_1, c_2} H_2\left(\frac{c_1(1-c_2)}{\nu^\text{disagree}_{c_1, c_2}}\right) \Big)  \\
                &= H_2(0.5) - \Big( 
                \nu^\text{agree}_{c_1, c_2} \frac{c_1c_2}{\nu^\text{agree}_{c_1, c_2}} \log \left(\frac{\nu^\text{agree}_{c_1, c_2}}{c_1c_2}\right) + \nu^\text{agree}_{c_1, c_2} \frac{(1-c_1)(1-c_2)}{\nu^\text{agree}_{c_1, c_2}} \log \left(\frac{\nu^\text{agree}_{c_1, c_2}}{(1-c_1)(1-c_2)}\right) \\
                &\hspace{2.2cm} + \nu^\text{disagree}_{c_1, c_2} \frac{c_1(1-c_2)}{\nu^\text{disagree}_{c_1, c_2}} \log\left(\frac{\nu^\text{disagree}_{c_1, c_2}}{c_1(1-c_2)}\right) + \nu^\text{disagree}_{c_1, c_2} \frac{(1-c_1)c_2}{\nu^\text{disagree}_{c_1, c_2}} \log\left(\frac{\nu^\text{disagree}_{c_1, c_2}}{(1-c_1)c_2}\right) \Big)  \\
                &= H_2(0.5) - \Big( 
                c_1c_2 \log \left(\frac{\nu^\text{agree}_{c_1, c_2}}{c_1c_2}\right) + (1-c_1)(1-c_2) \log \left(\frac{\nu^\text{agree}_{c_1, c_2}}{(1-c_1)(1-c_2)}\right) \\
                &\hspace{2.2cm} + c_1(1-c_2) \log\left(\frac{\nu^\text{disagree}_{c_1, c_2}}{c_1(1-c_2)}\right) + (1-c_1)c_2 \log\left(\frac{\nu^\text{disagree}_{c_1, c_2}}{(1-c_1)c_2}\right) \Big)  \\
    \end{align*}
    
    The first derivative is 
    \begin{align*}
        \frac{\mathrm{d}}{\mathrm{d}c_1} \phi_{c_2}(c_1) &= 0 - \Big( 
        c_2 \log \left(\frac{\nu^\text{agree}_{c_1, c_2}}{c_1c_2}\right) - \frac{c_2(1-c_2)}{\nu^\text{agree}_{c_1,c_2}}
        -(1-c_2) \log \left(\frac{\nu^\text{agree}_{c_1, c_2}}{(1-c_1)(1-c_2)}\right) + \frac{c_2(1-c_2)}{\nu^\text{agree}_{c_1,c_2}} \\
        &\hspace{1.7cm} + (1-c_2) \log\left(\frac{\nu^\text{disagree}_{c_1, c_2}}{c_1(1-c_2)}\right) - \frac{c_2(1-c_2)}{\nu^\text{disagree}_{c_1,c_2}}
        - c_2 \log\left(\frac{\nu^\text{disagree}_{c_1, c_2}}{(1-c_1)c_2}\right) + \frac{c_2(1-c_2)}{\nu^\text{disagree}_{c_1,c_2}} \Big)  \\
        &= - \Big( 
        c_2 \log \left(\frac{\nu^\text{agree}_{c_1, c_2}}{c_1c_2}\right) 
        -(1-c_2) \log \left(\frac{\nu^\text{agree}_{c_1, c_2}}{(1-c_1)(1-c_2)}\right) \\
        &\hspace{1.7cm} + (1-c_2) \log\left(\frac{\nu^\text{disagree}_{c_1, c_2}}{c_1(1-c_2)}\right) 
        - c_2 \log\left(\frac{\nu^\text{disagree}_{c_1, c_2}}{(1-c_1)c_2}\right) \Big)  \\
        &= - \Big( c_2\log(\nu^\text{agree}_{c_1, c_2}) - c_2 \log(c_1) -c_2 \log(1-c_2) \\
        & ~~~~~~~~~~ - (1-c_2)\log(\nu^\text{agree}_{c_1, c_2}) + (1-c_2)\log(1-c_1) + (1-c_2)\log(1-c_2) \\
        & ~~~~~~~~~~ + (1-c_2)\log(\nu^\text{disagree}_{c_1, c_2}) - (1-c_2)\log(c_1) - (1-c_2)\log(1-c_2) \\
        & ~~~~~~~~~~ - c_2\log(\nu^\text{disagree}_{c_1, c_2}) + c_2 \log(1-c_1) + c_2 \log(c_2) \Big)  \\
        &= \log\left(\frac{c_1}{1-c_1}\right) + (2c_2 - 1)\log\left( \frac{\nu^\text{disagree}_{c_1,c_2}}{\nu^\text{agree}_{c_1,c_2}} \right) \text{.}
    \end{align*}
    
    The second derivative is 
    \begin{align*}
        \frac{\mathrm{d}^2}{\mathrm{d}^2c_1} \phi_{c_2}(c_1) &= 
        \frac{1}{c_1(1-c_1)} + (2c_2-1)\frac{\nu^\text{agree}_{c_1,c_2}}{\nu^\text{disagree}_{c_1,c_2}} \cdot \frac{(2c_2-1)\nu^\text{disagree}_{c_1,c_2} - \nu^\text{agree}_{c_1,c_2}(1-2c_2)}{(\nu^\text{agree}_{c_1,c_2})^2} \\
        &= \frac{1}{c_1(1-c_1)} + (2c_2-1)^2\frac{1}{\nu^\text{disagree}_{c_1,c_2}} \cdot \frac{\nu^\text{disagree}_{c_1,c_2} + \nu^\text{agree}_{c_1,c_2}}{\nu^\text{agree}_{c_1,c_2}} \\
        &= \frac{1}{c_1(1-c_1)} + \frac{(2c_2-1)^2}{\nu^\text{disagree}_{c_1,c_2}\nu^\text{agree}_{c_1,c_2}} \\
        &= \frac{\nu^\text{disagree}_{c_1,c_2}\nu^\text{agree}_{c_1,c_2} + (2c_2-1)^2c_1(1-c_1) }{c_1(1-c_1)\nu^\text{disagree}_{c_1,c_2}\nu^\text{agree}_{c_1,c_2}} \\
        &= \frac{c_2(1-c_2)}{c_1(1-c_1)\nu^\text{disagree}_{c_1,c_2}\nu^\text{agree}_{c_1,c_2}} \text{.}
    \end{align*}

The second derivative is non-negative for $c_2 \in [0.5,1], c_1 \in [0.5,1)$.
Thus, $\sum_{c_2\in\Omega_{f_2}} f_2(c_2)\phi_{c_2}(c_1)$ is a convex scoring function. With Corollary~\ref{corollary:genAndSpec} follows the desired statement.
\end{proof}

In the last proposition, we assume that only the individual classifier's information is constraint but not their accuracy and look at the resulting ensemble information.

\begin{proposition}
    \textbf{(Information constrained specialists and generalists bound the ensemble information)}
    Consider $k$ classifiers with individual information $\information{i}$ and confidence distributions $f_i$ ($i \in \{1..k\}$). For each $i$, let the accuracies corresponding to the individual information be $\tilde{\accuracy{i}}~=~H_2^{-1}(1-\information{i})$. Let $f^\text{generalist}_{\overset{\sim}{\accuracy{i}}}$ and $f^\text{specialist}_{\overset{\sim}{\accuracy{i}}}$ as defined above. 
    Now consider the ensemble information based on the original classifiers, with ensemble confidence distribution $f_e~=~\bigotimes_{i = 1}^k\,f_i$ as well as ensembles with confidence distributions $\tilde{f}^\text{generalist}_e~=~\bigotimes_{i = 1}^k\,f^\text{generalist}_{\overset{\sim}{\accuracy{i}}}$ and  $\tilde{f}^\text{specialist}_e~=~\bigotimes_{i = 1}^k\,f^\text{specialist}_{\overset{\sim}{\accuracy{i}}}$ as in Theorem~\ref{theorem:IndAccIndInfoGroupAcc}. The information of the ensemble classifier is bounded by 
    $$\information{\tilde{f}^\text{generalist}_e} \leq \information{f_e} \leq \information{\tilde{f}^\text{specialist}_e}.$$  
\end{proposition}

\begin{proof}
    The ensemble information is
    \begin{align*}
        \information{f} = I\left(Y; O_e\right) &= \sum_{c\in\Omega_{f_e}} f_e(c)\left(H_2(0.5) - H_2(c)\right) \\
                &= \sum_{c_1\in\Omega_{f_1}} \sum_{c_2\in\Omega_{f_2}} f_1(c_1)f_2(c_2) \Big( 
                P(\hat{y_1} = \hat{y_2}|c_1, c_2) (1- H_2(P(\hat{y}_e\text{ correct}|\hat{y}_1 = \hat{y}_2, c_1, c_2 )))
                \\ &~~~~~~~~~~~~~~~~~~~~~~~~~~~~~~~~~~~~~~~~~~~~~~~~ + P(\hat{y_1} \neq \hat{y_2}|c_1, c_2) (1- H_2(P(\hat{y}_e\text{ correct}|\hat{y}_1 \neq \hat{y}_2, c_1, c_2))) \Big)  \\
    \end{align*}
    Our main work in this proof is to show that the function in the double sum is convex in $1-H_2(c_1)$ so that we can apply Corollary~\ref{corollary:genAndSpec} again. Let \\
    
    \begin{tabular}{p{1cm}lp{1cm}l}
        & $\nu^\text{agree}_{c_1, c_2} = c_1 \cdot c_2 +  (1-c_1) \cdot (1 - c_2)$
        & & (conditional probability to agree) \\
        & $\nu^\text{disagree}_{c_1, c_2} = c_1 \cdot (1 - c_2) +  c_1 \cdot (1 - c_2)$
        & & (conditional probability to disagree) \\
        & $\eta^\text{agree}_{c_1, c_2} = \frac{c_1 \cdot c_2}{\nu^\text{agree}}$
        & & (conditional confidence upon agreement) \\
        & $\eta^\text{disagree}_{c_1, c_2} = \frac{\max\{ c_1 \cdot (1 - c_2),  c_1 \cdot (1 - c_2) \}}{\nu^\text{disagree}_{c_1, c_2}}$
        & & (conditional confidence upon disagreement)
    \end{tabular} \\[0.5cm]
    Also denote the local information by $\iota(c) = H_2(0.5)-H_2(c)$. The relevant term in the big brackets is 
    \begin{align*}
        \phi_{c_2}(c_1) = \Big( 
                \nu^\text{agree}_{c_1, c_2} \iota(\eta^\text{agree}_{c_1, c_2})
                + \nu^\text{disagree}_{c_1, c_2} \iota(\eta^\text{disagree}_{c_1, c_2}) \Big)  \\
    \end{align*}
    We will show that $\phi_{c_2}(c_1)$ is convex in $\iota(c_1)$ by showing that the second derivative is non-negative.

    \begin{align*}
        \frac{\mathrm{d}^2 \phi_{c_2}(c_1)}{ \mathrm{d} \iota(c_1)^2}
        &= \frac{\mathrm{d}}{ \mathrm{d} \iota(c_1)} \frac{\mathrm{d} \phi_{c_2}(c_1)}{ \mathrm{d} \iota(c_1)} \\
        &= \frac{\mathrm{d}}{ \mathrm{d} \iota(c_1)} \left( \frac{ \frac{\mathrm{d} \phi_{c_2}(c_1)}{\mathrm{d} c_1} }{ \frac{ \mathrm{d} \iota(c_1) }{ \mathrm{d} c_1 }} \right) \\
        &= \frac{\phi''_{c_2} \iota' - \phi'_{c_2} \iota'' }{ ( \iota' ) ^3} \\
    \end{align*}    
    We had already derived $\iota'$, $\iota''$, $\phi'_{c_2}{c_1}$ and $\phi''_{c_2}{c_1}$ in the proof of Proposition~\ref{prop:IndAccGroupInfo}:
    \begin{align*}
        \iota'(c) = \frac{\mathrm{d}}{\mathrm{d}c} \iota(c) &= \log\left( \frac{c}{1-c} \right) \\
        \iota''(c) = \frac{\mathrm{d}^2}{\mathrm{d}c^2} \iota(c) &= \frac{1}{c(1-c)} \\
        \phi'_{c_2}{c_1} = \frac{\mathrm{d}}{\mathrm{d}c_1} \phi_{c_2}(c_1) 
        &= \log\left(\frac{c_1}{1-c_1}\right) + (2c_2 - 1)\log\left( \frac{\nu^\text{disagree}_{c_1,c_2}}{\nu^\text{agree}_{c_1,c_2}} \right) \\
        \phi''_{c_2}{c_1} = \frac{\mathrm{d}^2}{\mathrm{d}^2c_1} \phi_{c_2}(c_1) 
         &= \frac{c_2(1-c_2)}{c_1(1-c_1)\nu^\text{disagree}_{c_1,c_2}\nu^\text{agree}_{c_1,c_2}} \text{.}
    \end{align*}
    
    Now we can put together the second derivative of $\phi_{c_2}(c_1)$ with respect to $\iota(c_1)$. For convexity, we want to show that this is non-negative.
    \begin{align*}
        &\frac{\mathrm{d}^2 \phi_{c_2}(c_1)}{ \mathrm{d} \iota(c_1)^2} \geq 0 \\
        &\iff \frac{\phi''_{c_2} \iota' - \phi'_{c_2} \iota'' }{ ( \iota' ) ^3} \geq 0 \\
        &\overset{(1)}{\iff} \phi''_{c_2} \iota' - \phi'_{c_2} \iota'' \geq 0 \\
        &\iff \frac{c_2(1-c_2)}{c_1(1-c_1)\nu^\text{disagree}_{c_1,c_2}\nu^\text{agree}_{c_1,c_2}} \log\left(\frac{c_1}{1-c_1}\right) - \left(\log\left(\frac{c_1}{1-c_1}\right) + (2c_2 - 1)\log\left( \frac{\nu^\text{disagree}_{c_1,c_2}}{\nu^\text{agree}_{c_1,c_2}} \right) \right)\frac{1}{c_1(1-c_1)} \geq 0 \\
        &\overset{(2)}{\iff} \frac{c_2(1-c_2)}{\nu^\text{disagree}_{c_1,c_2}\nu^\text{agree}_{c_1,c_2}} \log\left(\frac{c_1}{1-c_1}\right) - \left(\log\left(\frac{c_1}{1-c_1}\right) + (2c_2 - 1)\log\left( \frac{\nu^\text{disagree}_{c_1,c_2}}{\nu^\text{agree}_{c_1,c_2}} \right) \right) \geq 0 \\
        &\iff \log\left(\frac{c_1}{1-c_1}\right) \left( \frac{c_2(1-c_2)}{\nu^\text{disagree}_{c_1,c_2}\nu^\text{agree}_{c_1,c_2}}  - 1 \right) - (2c_2 - 1)\log\left( \frac{\nu^\text{disagree}_{c_1,c_2}}{\nu^\text{agree}_{c_1,c_2}} \right)  \geq 0 \\
        &\iff \log\left(\frac{c_1}{1-c_1}\right) \left( \frac{c_2(1-c_2) - \nu^\text{disagree}_{c_1,c_2}\nu^\text{agree}_{c_1,c_2}}{\nu^\text{disagree}_{c_1,c_2}\nu^\text{agree}_{c_1,c_2}} \right) - (2c_2 - 1)\log\left( \frac{\nu^\text{disagree}_{c_1,c_2}}{\nu^\text{agree}_{c_1,c_2}} \right)  \geq 0 \\
        &\iff \log\left(\frac{c_1}{1-c_1}\right) \left( \frac{(2c_2 -1 )^2c_1(c_1-1)}{\nu^\text{disagree}_{c_1,c_2}\nu^\text{agree}_{c_1,c_2}} \right) - (2c_2 - 1)\log\left( \frac{\nu^\text{disagree}_{c_1,c_2}}{\nu^\text{agree}_{c_1,c_2}} \right)  \geq 0 \\
        &\overset{(3)}{\iff} \log\left(\frac{c_1}{1-c_1}\right) \left( \frac{(2c_2 -1 )c_1(c_1-1)}{\nu^\text{disagree}_{c_1,c_2}\nu^\text{agree}_{c_1,c_2}} \right) - \log\left( \frac{\nu^\text{disagree}_{c_1,c_2}}{\nu^\text{agree}_{c_1,c_2}} \right)  \geq 0 \\
        &\iff \log\left( \frac{\nu^\text{agree}_{c_1,c_2}}{\nu^\text{disagree}_{c_1,c_2}} \right) - \log\left(\frac{c_1}{1-c_1}\right) \left( \frac{(2c_2 -1 )c_1(1-c_1)}{\nu^\text{disagree}_{c_1,c_2}\nu^\text{agree}_{c_1,c_2}} \right)  \geq 0 \\
    \end{align*}
    In (1), (2) and (3) we multiply with $(\iota')^3$, $\frac{1}{c_1(1-c_1)}$ and $\frac{1}{2c_2 - 1}$, respectively. These terms are larger than 0 in $c_1, c_2 \in (0.5, 1)$ so that the inequality sign does not change. It remains to show that
    \begin{align*}
        \omega(c_1, c_2) := \log\left( \frac{\nu^\text{agree}_{c_1,c_2}}{\nu^\text{disagree}_{c_1,c_2}} \right) - \log\left(\frac{c_1}{1-c_1}\right) \left( \frac{(2c_2 -1 )c_1(1-c_1)}{\nu^\text{disagree}_{c_1,c_2}\nu^\text{agree}_{c_1,c_2}} \right) \geq 0 \text{.}
    \end{align*}
    We will do so by switching to the partial derivative with respect to $c_2$ (instead of $c_1$ as above) and demonstrate that (i) $\omega(c_1,c_2) = 0$ for $c_2 = 0.5$ and $c_2 = 1$, (ii) $\frac{\partial \omega(c_1, c_2)}{\partial c_2}\Bigr|_{\substack{c_2=0.5}} \geq 0$, and (iii) $\frac{\partial \omega(c_1, c_2)}{\partial c_2} = 0$ for only one $c_2\in(0.5,1)$. 
    
    (i) 
    \begin{align*}
        \omega(c_1, 0.5) &= \log\left( \frac{0.5}{0.5} \right) - \log\left(\frac{c_1}{1-c_1}\right) \left( \frac{(2\cdot 0.5 -1 )c_1(1-c_1)}{0.5\cdot 0.5} \right) = 0 - 0 = 0 \\
        \omega(c_1, 1) &= \log\left( \frac{c_1}{1-c_1} \right) - \log\left(\frac{c_1}{1-c_1}\right) \left( \frac{(2\cdot 1 -1 )c_1(1-c_1)}{c_1(1-c_1) } \right) = \log\left( \frac{c_1}{1-c_1} \right) - \log\left(\frac{c_1}{1-c_1}\right) = 0 
    \end{align*}
    
    (ii) 
    \begin{align*}
        \frac{\partial \omega(c_1, c_2)}{\partial c_2}
        &= \frac{\nu^\text{disagree}_{c_1,c_2}}{\nu^\text{agree}_{c_1,c_2}} \cdot \frac{(2c_1 - 1)\nu^\text{disagree}_{c_1,c_2} - \nu^\text{agree}_{c_1,c_2}(-1)(2c_1-1)}{(\nu^\text{disagree}_{c_1,c_2})^2} \\
        &~~~~ - 
        \log\left(\frac{c_1}{1-c_1}\right)c_1(1-c_1) \left( \frac{2\nu^\text{agree}_{c_1,c_2}\nu^\text{disagree}_{c_1,c_2} - (2c_2-1)(-1)(2c_2-1)(2c_1-1)^2}{(\nu^\text{agree}_{c_1,c_2}\nu^\text{disagree}_{c_1,c_2})^2} \right) \\
        &= \frac{2c_1 - 1}{\nu^\text{agree}_{c_1,c_2}\nu^\text{disagree}_{c_1,c_2}} - 
        \log\left(\frac{c_1}{1-c_1}\right)c_1(1-c_1) \left( \frac{2\nu^\text{agree}_{c_1,c_2}\nu^\text{disagree}_{c_1,c_2} + (2c_2-1)^2(2c_1-1)^2}{(\nu^\text{agree}_{c_1,c_2}\nu^\text{disagree}_{c_1,c_2})^2} \right) \\
        \frac{\partial \omega(c_1, c_2)}{\partial c_2}\Bigr|_{\substack{c_2=0.5}}
        &= \frac{2c_1 - 1}{0.5\cdot 0.5} - \log\left(\frac{c_1}{1-c_1}\right)c_1(1-c_1) \left( \frac{2\cdot 0.5\cdot 0.5 + (2\cdot 0.5-1)^2(2c_1-1)^2}{(0.5\cdot 0.5)^2} \right) \\
        &= (8c_1 - 4) - 8c_1(1-c_1)\log\left( \frac{c_1}{1-c_1} \right) \\
    \end{align*}
    To finish (ii) we have to show that $\overset{\sim}{\omega}(c_1) := \frac{\partial \omega(c_1, c_2)}{\partial c_2}\Bigr|_{\substack{c_2=0.5}} \geq 0$. We use a similar strategy as before: We show that (ii.i) $\overset{\sim}{\omega}(0.5) = 0$ and that (ii.ii) $\frac{\mathrm{d}\overset{\sim}{\omega}(c_1)}{\mathrm{d}c_1} \geq 0$.
    
    (ii.i)
    \begin{align*}
        \overset{\sim}{\omega}(0.5) = (8\cdot0.5 - 4) - 8\cdot 0.5\cdot(1-0.5) \log\left( \frac{0.5}{1-0.5} \right) = 0 - 2\cdot 0 = 0
    \end{align*}
    
    (ii.ii)
    \begin{align*}
        \frac{\mathrm{d} \overset{\sim}{\omega}(c_1)}{\mathrm{d} c_1} 
        &= 8 - 8\left( (1-2c_1)\log\left( \frac{c_1}{1-c_1} \right) + c_1(1-c_1)\frac{1-c_1}{c_1}\frac{1\cdot(1-c_1) - (-c_1)}{(1-c_1)^2} \right) \\
        &= 8 - 8\left( (1-2c_1)\log\left( \frac{c_1}{1-c_1} \right) +1 \right) \\
        &= 8 (2c_1 - 1) \log \left( \frac{c_1}{1-c_1}  \right) \geq 0
    \end{align*}
    
    Taken together, $\overset{\sim}{\omega}(c_1)$ starts non-negative at $c_1= 0.5$ (ii.i) and only increases for larger $c_1$ (ii.ii). This finishes (ii) showing that
    \begin{align*}
        \frac{\partial \omega(c_1, c_2)}{\partial c_2}\Bigr|_{\substack{c_2=0.5}}
        = (8c_1 - 4) - 8c_1(1-c_1)\log\left( \frac{c_1}{1-c_1} \right) \geq 0 \\
    \end{align*}
    
    (iii) 
    \begin{align*}
        &\frac{\partial \omega(c_1, c_2)}{\partial c_2} \overset{!}{=} 0 \\
        \iff& \frac{2c_1 - 1}{\nu^\text{agree}_{c_1,c_2}\nu^\text{disagree}_{c_1,c_2}} - 
        \log\left(\frac{c_1}{1-c_1}\right)c_1(1-c_1) \left( \frac{2\nu^\text{agree}_{c_1,c_2}\nu^\text{disagree}_{c_1,c_2} + (2c_2-1)^2(2c_1-1)^2}{(\nu^\text{agree}_{c_1,c_2}\nu^\text{disagree}_{c_1,c_2})^2} \right) \overset{!}{=} 0 \\
        \overset{(1)}{\iff}& (2c_1 - 1)\nu^\text{agree}_{c_1,c_2}\nu^\text{disagree}_{c_1,c_2} - 
        \log\left(\frac{c_1}{1-c_1}\right)c_1(1-c_1) \left( 2\nu^\text{agree}_{c_1,c_2}\nu^\text{disagree}_{c_1,c_2} + (2c_2-1)^2(2c_1-1)^2 \right) \overset{!}{=} 0 \\
        \iff& \left( (-8c_1^3 + 12c_1^2 - 6 c_1 + 1) - c_1(1-c_1)\log\left(\frac{c_1}{1-c_1}\right)(8c_1^2 - 8c_1 + 2) \right)c_2^2 \\
        & -\left( (-8c_1^3 + 12c_1^2 - 6 c_1 + 1) - c_1(1-c_1)\log\left(\frac{c_1}{1-c_1}\right)(8c_1^2 - 8c_1 + 2) \right)c_2 \\
        & + \left( (-2c_1^3 + 3c_1^2 - c_1) - c_1(1-c_1)\log\left(\frac{c_1}{1-c_1}\right) (2 c_1^2 -2c_1 +1) \right)  \overset{!}{=} 0 \\
        \overset{(2)}{\iff}& c_2^2 - c_2 + \frac{\left( (-2c_1^3 + 3c_1^2 - c_1) - c_1(1-c_1)\log\left(\frac{c_1}{1-c_1}\right) (2 c_1^2 -2c_1 +1) \right)}{\left( (-8c_1^3 + 12c_1^2 - 6 c_1 + 1) - c_1(1-c_1)\log\left(\frac{c_1}{1-c_1}\right)(8c_1^2 - 8c_1 + 2) \right)} \overset{!}{=} 0 \\
    \end{align*}
    
    At (1) we multiply with $(\nu^\text{agree}_{c_1,c_2}\nu^\text{disagree}_{c_1,c_2})^2$ and in (2) we divide by the constant in the denominator. Both are larger than 0 in $c_1, c_2 \in(0.5, 1)$. The result is a quadratic equation that has at most one zero in $c_2\in(0.5, 1)$. This completes (iii).
    
    Taken together, for any $c_1 \in (0.5, 1)$, $\omega(c_1, c_2)$ is zero at the corner cases $c_2 = 0.5$ and $c_2 = 1$ (i), increases from $c_2=0.5$ on (ii) and only changes monotonicity once (iii) so that $\omega(c_1, c_2)$ is non-negative.
    
    With that, $\phi_{c_2}$ is convex in $\iota(c_1)$. Lemma~\ref{lemma:jensens} yields the desired result.
    
\end{proof}

\end{document}